\documentclass{article}

\usepackage[preprint]{neurips_2026}
\usepackage[utf8]{inputenc}
\usepackage[T1]{fontenc}
\usepackage[american]{babel}
\usepackage{hyperref}
\usepackage{url}
\usepackage{graphicx}
\usepackage{mathtools}
\usepackage{amssymb}
\usepackage{array}
\usepackage{booktabs}

\usepackage{nicefrac}
\usepackage{microtype}
\usepackage{xcolor}

\hypersetup{%
  pdftitle={Conditioning Tree-Based Diffusions and Flows for Probabilistic Tabular Regression},
  pdfauthor={Silas Koemen},
  pdfsubject={Probabilistic tabular regression with gradient-boosted diffusion and flow models},
  pdfkeywords={diffusion models, flow matching, gradient boosting, probabilistic regression, tabular data},
}

\newcommand{\crps}{\mathrm{CRPS}}
\newcommand{\crpss}{\mathrm{CRPSS}}

\newcommand{\E}{\mathbb{E}}
\newcommand{\Var}{\mathrm{Var}}
\newcommand{\Norm}{\mathcal{N}}
\newcommand{\dd}{\mathrm{d}}

\newtheorem{lemma}{Lemma}
\newtheorem{proposition}{Proposition}

\title{Conditioning Tree-Based Diffusions and Flows for Probabilistic Tabular Regression}

\author{%
  Silas Koemen\thanks{ORCID: \href{https://orcid.org/0009-0004-2843-9362}{0009-0004-2843-9362}.}\\
  Independent Researcher\\
  \texttt{research@silaskoemen.com}
}

\begin{document}
\maketitle

\begin{abstract}
Tree-based diffusion models fit flexible conditional predictive distributions for tabular regression without a neural density estimator, but they inherit their design defaults---noising path, parameterization, training distribution, features, sampler---from the neural setting. We show these defaults are the binding constraint: what a gradient-boosted ensemble actually solves is a supervised regression problem whose conditioning they determine. We present DiffGBM, which makes them explicit along two axes. First, a Gaussian-path flow-matching trainer for $p(y \mid x)$ that learns a velocity field directly and recovers the score algebraically, admitting few-step deterministic ODE sampling. Second, we expose the score-side recipe---residualization, EDM-style preconditioning, log-sigma time sampling, noise-level features, loss weighting, and histogram resolution---as jointly tunable axes over a shared LightGBM surface rather than one frozen bundle. This \emph{score-flex} space represents the published recipe as a special case; across eleven tabular benchmarks under fold-0 tuning, folds-1--5 evaluation, and a matched 40-trial budget and sampler, the selected configurations beat that baseline on \emph{every} dataset (paired Wilcoxon $11/0$, $p<10^{-3}$), with the best aggregate CRPS skill (0.725 vs.\ 0.699) of any row. The two rows are complementary: score-flex buys accuracy with a stochastic sampler and is the slowest row, while flow matching is the cheapest sampler ($5.2\times$ faster than the published baseline) and the best-calibrated DiffGBM row. Tuned non-diffusion baselines still win individual datasets, and stochastic ($\varepsilon>0$) flow samplers do not Pareto-dominate the deterministic corner.
\end{abstract}

\section{Introduction}\label{sec:intro}
Probabilistic regression on tabular data must combine the predictive strength of gradient-boosted trees with full predictive distributions usable for calibration, decision-making, and uncertainty quantification. Treeffuser~\citep{treeffuser} achieves this by training LightGBM~\citep{ke2017lightgbm} regressors to estimate the score of a noise-perturbed conditional density and integrating the corresponding reverse-time stochastic differential equation (SDE)~\citep{song2021score} at inference. The framework can produce well-calibrated predictive distributions when the score is estimated successfully, while avoiding neural-network density estimators. We call our extension \emph{DiffGBM}: it keeps this tree-based backbone but treats the noising path, parameterization, training distribution, features, and sampler as tunable conditioning choices rather than fixed defaults.

Our notion of tractability is computational rather than likelihood-exact: predictive quantities are Monte Carlo estimates, but training reduces to supervised tree regression, the final FM row samples with a five-step Heun ODE, and Eq.~\eqref{eq:score} recovers scores when stochastic sampling is desired.

We develop DiffGBM along two axes: score-side modifications matched to histogram-bin tree learners, and Gaussian-path flow matching~\citep{lipman2023flow,liu2023rectified}, which removes the indirect score-to-drift mapping and opens access to stochastic interpolants~\citep{albergo2023stochastic} with non-negative stochasticity schedule $\varepsilon(t)$.

\paragraph{Contributions.} We make three contributions:
\begin{enumerate}
  \item \emph{Flow matching for conditional tree-based prediction:} a Gaussian-path FM trainer for $p(y \mid x)$ that learns a velocity field directly with gradient-boosted trees, with algebraic score recovery when stochastic-interpolant sampling or endpoint-error analysis requires it (\S\ref{sec:fm}, Lemma~\ref{lem:wronskian-score}). This yields the cheapest sampler and the best-calibrated DiffGBM operating point.
  \item \emph{A jointly tunable score recipe for tree-based diffusion:} conditional-mean residualization, EDM preconditioning, log-$\sigma$ time sampling, an explicit log-$\sigma$ feature, loss weighting, and histogram resolution, exposed as recipe axes over the shared LightGBM surface (\S\ref{sec:score-side}). The resulting score-flex space represents the published recipe as a special case and, tuned per dataset at a matched budget and sampler, beats it on every benchmark (\S\ref{sec:headline}).
  \item \emph{Regression conditioning as the shared bottleneck:} an account of why both sets of gains appear, grounded in the finite split resolution available to a histogram tree rather than in population-level properties of the objectives (\S\ref{sec:conditioning}, \S\ref{sec:disc}, Appendix~\ref{app:theory}).
\end{enumerate}

Across eleven benchmarks, with one fold for tuning and five for evaluation, the jointly tuned score-flex recipe gives the best aggregate CRPS skill score of any row and beats the published baseline on every dataset, while flow matching gives the cheapest sampler and the best-calibrated DiffGBM operating point. The shared lesson is that noising path, parameterization, training distribution, features, histogram resolution, and sampler must make the induced supervised problem well matched to histogram splits, and that these are best exposed as tunable axes rather than fixed as one bundle.

\section{Related Work}\label{sec:related}

\paragraph{Probabilistic prediction with tree ensembles.} Quantile regression forests~\citep{meinshausen2006qrf} and quantile-regression boosting estimate conditional quantiles directly, fitting one model per quantile level; they remain strong baselines and we tune one here. Distributional methods instead fit a parametric predictive law per input---NGBoost~\citep{duan2020ngboost} boosts distributional parameters under the natural gradient, and CatBoost's uncertainty mode~\citep{prokhorenkova2018catboost} predicts a mean and a variance---so their sharpness is bounded by the assumed family. Nonparametric alternatives keep the predictive law implicit: iBUG~\citep{brophy2022ibug} builds a predictive distribution from the $k$ nearest training instances in tree-affinity space, and distributional random forests~\citep{cevid2022drf} estimate conditional distributions through forest weights. A separate line targets uncertainty quantities other than the sharpness-and-coverage regime studied here: SGLB~\citep{ustimenko2021sglb} uses Langevin boosting for epistemic uncertainty and out-of-distribution detection; Boulevard~\citep{zhou2022boulevard} yields asymptotically valid intervals without a full generative predictive law; and conformal wrappers such as conformalized quantile regression~\citep{romano2019cqr} give finite-sample marginal coverage on top of any base learner, which is orthogonal to---and composable with---the generative question we study. Our benchmark therefore covers the sample-based and quantile families (\S\ref{sec:exp}) and treats the remaining three as answering different questions.

\paragraph{Diffusion and flow matching with trees.} Treeffuser~\citep{treeffuser}, described in \S\ref{sec:intro}, is the direct predecessor and the baseline we extend. Forest-Diffusion~\citep{jolicoeurmartineau2024forest} independently established that gradient-boosted trees can carry a velocity field as well as a score, training XGBoost regressors under both score-based diffusion and conditional flow matching. Its target differs from ours in three ways that shape the design. It models the joint distribution over all columns for unconditional row generation and imputation rather than a conditional predictive law $p(y \mid x)$ scored by proper scoring rules; it trains a separate model per variable and per noise level ($p \times n_t$ models, with the training set duplicated across levels), whereas we fit one regressor per response coordinate with $t$ as a feature, which is what makes the noise-level training distribution, the noise-level features, and the loss weighting design axes here; and it does not use a velocity-to-score identity to move a velocity-trained model into score-based or stochastic-interpolant samplers. We take tree-based flow matching as established and ask instead what conditioning the induced regression problem needs for conditional prediction, and which of those choices should be fixed versus tuned. On the neural side, TabDDPM~\citep{kotelnikov2023tabddpm} applies diffusion to tabular generation and CARD~\citep{han2022card} to conditional regression; CARD is a baseline in \S\ref{sec:exp}.

\paragraph{Design axes for diffusion and flow models.} The axes we expose are established in the neural literature. EDM~\citep{karras2022edm} treats preconditioning, noise-level parameterization, and the training-time noise distribution as explicit design choices; min-SNR-$\gamma$~\citep{hang2023minsnr} reweights the loss across noise levels; flow matching~\citep{lipman2023flow} and rectified flow~\citep{liu2023rectified} learn velocity fields on Gaussian paths; and stochastic interpolants~\citep{albergo2023stochastic} unify deterministic and stochastic samplers for those paths. Our contribution is not a new axis but the finding that good settings for a histogram-split tree learner differ from the neural defaults, and differ from each other across datasets, which is why we tune them jointly.

\section{Method}\label{sec:method}

\subsection{Score-Side Modifications}\label{sec:score-side}
The published Treeffuser uses VESDE marginals with uniform $t$ sampling, a noise-prediction target, and LightGBM features $[\,y_t,\,X,\,t\,]$. We change four pieces: response/input conditioning, allocation of training rows over noise levels, explicit noise-level features, and sampler efficiency.

\paragraph{Conditional-mean residualization.} We fit a cross-validated LightGBM mean predictor $\hat{\mu}(x)$ and apply diffusion to $y - \hat{\mu}(x)$ rather than $y$, moving conditional location out of the score-regression target. On the few-step FM ODE, residualization improves aggregate CRPS and 90\% coverage error decisively (81/100 paired wins on $|\text{cE}|$@90; Appendix~\ref{app:off-vs-mean}). The FM row uses a fixed high-capacity residualizer (configuration C), selected from the diagnostic sweep of Appendix~\ref{app:resid-config} rather than tuned jointly with the flow hyperparameters; that sweep's capacity frontier is flat, spanning CRPSS $0.509$--$0.523$. On large datasets a fixed residualizer can become a bottleneck once the inner model has enough data to represent the conditional mean directly: removing it roughly halves the frozen score$^+$ deficit on \texttt{ct\_slices}, while FM depends on it (Appendix~\ref{app:large-data}). Tuning residualization jointly with the rest of the score recipe closes that gap, so the headline score model treats it as a tunable axis rather than a fixed choice.

\paragraph{EDM preconditioning.} Following \citet{karras2022edm}, we predict a denoised target $D_\theta(y_t, t)$ scaled to unit variance across noise levels and reconstruct the score $s(y_t, t) = (D_\theta - y_t)/\sigma(t)^2$. This stabilizes the regression target for trees, which would otherwise need to extrapolate noise-prediction magnitudes across $t$.

\paragraph{Log-sigma time sampling.} We draw $\log\sigma(t) \sim \Norm(-1.2, 1.2^2)$, clip to the SDE's achievable range, and invert. Because trees bin features by data density, the training-time distribution directly controls where the score model spends capacity.

\paragraph{Log-sigma noise feature.} The feature vector for LightGBM becomes $[\,y_t,\,X,\,t,\,\log\sigma(t)\,]$, giving the regressor an explicit noise-level signal aligned with EDM scaling.

\subsection{Flow Matching with Gaussian Paths}\label{sec:fm}
A Gaussian flow path interpolates data $y_0$ and a standard-normal prior $z$ via
\begin{equation}\label{eq:path}
y_t = \alpha(t)\, y_0 + \beta(t)\, z, \quad t \in [0,1],
\end{equation}
with boundary conditions $\alpha(0)=1$, $\alpha(1)=0$, $\beta(0)=0$, $\beta(1)=1$. We compare three paths: \emph{linear} ($\alpha = 1-t$, $\beta = t$), \emph{trigonometric} ($\alpha = \cos(\pi t/2)$, $\beta = \sin(\pi t /2)$), and a \emph{VP} schedule with linear $\beta_t$ giving $\alpha^2(t) = \exp(-T(t))$, $T(t) = \tfrac{1}{2}\beta_{\min} t + \tfrac{1}{4}(\beta_{\max} - \beta_{\min}) t^2$. The linear and trigonometric paths satisfy $\alpha(1)=0$ exactly; the VP schedule satisfies it only in the limit, giving $\alpha(1)=0.081$ at the $\beta_{\min}=0.1$, $\beta_{\max}=20$ setting we use, so its standard-normal initialization is an endpoint approximation, as in the standard discrete VP/DDPM schedule. The training target is the conditional velocity $u_t = \alpha'(t) y_0 + \beta'(t) z$, regressed against the model's $v_\theta(y_t, x, t)$.

\paragraph{Velocity-to-score formula.} For any Gaussian path with Wronskian $W(t) = \alpha(t) \beta'(t) - \alpha'(t) \beta(t)$, the conditional score is algebraically recoverable from the population velocity:
\begin{equation}\label{eq:score}
\nabla \log p_t(y_t \mid x) \;=\; \frac{\alpha'(t)\, y_t - \alpha(t)\, v_\theta(y_t, x, t)}{W(t)\,\beta(t)}.
\end{equation}
For linear FM this reduces to $-(y_t + (1-t) v)/t$; for trig to $-y_t - (2/\pi)\cot(\pi t/2) v$.
We use this identity operationally: it lets a velocity-trained DiffGBM model enter the same stochastic-interpolant sampler family as score matching and exposes the endpoint amplification studied in Appendix~\ref{app:theory}.

\paragraph{Stochastic-interpolant family.} \citet{albergo2023stochastic} show that any non-negative $\varepsilon(t)$ defines a marginal-preserving reverse-time sampler at the population velocity/score:
\begin{equation}\label{eq:sde}
\dd y_\tau = \left(-v_\theta + \tfrac{\varepsilon^2}{2}\, s_\theta\right) \dd \tau + \varepsilon\, \dd W_\tau,
\end{equation}
where $\tau = 1 - t$ is reverse time and $s_\theta$ is the score from \eqref{eq:score}. The choice $\varepsilon \equiv 0$ recovers the deterministic probability-flow ODE. At the true velocity and score, every non-negative $\varepsilon(t)$ preserves the marginals and only randomizes the trajectories; with an approximate model it can shift the realized dispersion in either direction, so we treat it as a calibration knob. We use the schedule $\varepsilon(t) = c\, t$, which vanishes at the data endpoint and suppresses recovered-score error in the stochastic drift (see Proposition~\ref{prop:endpoint-amp}).

\paragraph{Why these choices suit trees.} Eq.~\eqref{eq:score} makes FM compatible with stochastic-interpolant samplers; variance-preserving paths keep the tree input $y_t$ near unit scale, unlike a linear path with $\Var(y_t)\approx(1-t)^2+t^2$; and $\varepsilon=0$ avoids recovered-score error in the stochastic drift (Appendix~\ref{app:theory}). Path choice is therefore a finite-regression and sampler tradeoff rather than a settled hierarchy.

\subsection{Why Conditioning, and Not the Objective}\label{sec:conditioning}
For a fixed Gaussian interpolation path, the score and flow objectives are equivalent at the population level: Eq.~\eqref{eq:score} maps between them exactly. Different paths induce different $p_t$ but share the $t=0$ endpoint, so no choice among the objectives or paths we compare changes the target predictive law. What differs is the finite supervised problem each one hands to LightGBM, and the constraint that binds there is split resolution. A histogram tree does not see $y_t$; it sees a bucketed version of $y_t$, and in our trainer rows from every noise level are pooled into one feature matrix before binning. A single set of bin edges must therefore serve every $t$ at once.

Proposition~\ref{prop:histogram-bins} (Appendix~\ref{app:theory}) makes the cost of that compromise precise in a one-feature proxy: the error of a pooled histogram representation carries a term $\int p_t / \bar p^{\,2}$, which blows up wherever the pooled bin density $\bar p$ allocates little resolution to a region where an individual noise level $p_t$ still places mass. A VE path creates exactly this situation, because its tree input spans data scale at low noise and expands by orders of magnitude at high noise; one shared binning cannot resolve both regimes. This is a statement about finite bin allocation, not about the smoothness of the fixed-$t$ regression function---Figure~\ref{fig:toy-preconditioning}(c) fits a fresh LightGBM at each fixed $t$ and finds the VE score and velocity curves indistinguishable.

Read this way, the modifications of \S\ref{sec:score-side} and the flow paths of \S\ref{sec:fm} do the same job by different means. EDM input preconditioning rescales the tree input toward a common range across $t$; variance-preserving flow paths keep it near unit scale by construction; log-$\sigma$ sampling and the explicit log-$\sigma$ feature control where capacity is spent and let splits condition on the noise regime directly; residualization removes predictable conditional location from the target. Figure~\ref{fig:toy-preconditioning}(b) shows the effect empirically. This also predicts the axis on which the story is not universal: histogram resolution itself is a lever, so a dataset whose difficulty is concentrated at fine scales can prefer more bins over better preconditioning---which is what we observe on \texttt{ct\_slices} (\S\ref{sec:headline}, Appendix~\ref{app:large-data}).

\subsection{Implementation Notes}\label{sec:impl}
The FM code reuses Treeffuser's per-output-dimension LightGBM fitter and feature pipeline; only the target and prior change. Sampling uses the same Heun integrator and residualizer as the score path, and both training and sampling avoid the singular endpoint with $t \geq 10^{-5}$. Because the learned score or velocity components are ordinary LightGBM regressors, native LightGBM handling of categorical features, missing-value splits, and constraints can be used through the same feature pipeline and parameter pass-through.

\section{Experiments}\label{sec:exp}

\paragraph{Datasets and protocol.} Eleven standard tabular regression benchmarks (Table~\ref{tab:per-dataset})---nine from the UCI repository, including the large CT-slice localization dataset, and two from scikit-learn---are evaluated with a nested tuning protocol. We split each dataset into six folds: fold~0 is the only fold used for hyperparameter selection, and folds~1--5 are evaluation folds. For each dataset/model family, Optuna tunes on the fold-0 train/validation split; the selected configuration is then refit and evaluated on each held-out evaluation fold. Fold~0 is never a test fold, though it does enter the training set of the models evaluated on folds~1--5. Evaluation-fold training rows range from 256 on yacht to 44{,}583 on CT slices. For sample-based methods, each test point uses 200 samples; DiffGBM sampler choices are detailed in Appendix~\ref{app:datasets}. Each predictive metric is therefore a Monte Carlo estimate, but it is averaged over all held-out test points and five evaluation folds, so the residual sampling error in the aggregate CRPS and coverage numbers is small relative to the dataset-level standard errors reported in Table~\ref{tab:robustness-ranks}. Metrics are averaged over evaluation folds and, for headline summaries, over datasets.

\paragraph{Metrics.} We report CRPS skill score $\crpss = 1 - \crps_{\text{model}} / \crps_{\text{climatology}}$, mean rel-CRPS, DSS~\citep{dawid1984pit}, PIT-KS pass rate at $\alpha=0.05$, absolute coverage error at 50/90/95\% nominal levels~\citep{gneiting2007crps}, and sample-generation time. CRPSS and rel-CRPS are the aggregate comparison metrics because raw CRPS is target-scale dependent; mean climatological CRPS ranges from 0.0043 on \texttt{naval} to 44.8 on \texttt{diabetes}. Timings are measured on one Apple-arm64 workstation, exclude fitting, and include generating 200 samples per test point plus inverse transformations. They use practical benchmark settings rather than a single-thread cap: DiffGBM samples in batches with \texttt{n\_parallel=20} and LightGBM \texttt{n\_jobs=-1}; other libraries use their configured parallel defaults.

\paragraph{Variants compared.} We report two DiffGBM rows against the published Treeffuser baseline. \textsc{published} is the original recipe (noise prediction, uniform $t$, no residualization) tuned per dataset. \textsc{DiffGBM-score-flex} exposes the score-side recipe axes of \S\ref{sec:score-side}---score parameterization, noise-level features, $t$ sampling and its log-$\sigma$ prior, loss weighting, residualization, and histogram resolution---as jointly tunable dimensions over the shared LightGBM surface, under the same 50-step Euler SDE sampler as the published baseline, so the two rows are matched on both budget and sampler. The space contains the published recipe as a special case. A PF-ODE twin of the same recipe space is tuned separately and reported as an additional operating point (Appendix~\ref{app:large-data}); it trades CRPS for tighter interval coverage. \textsc{DiffGBM-FM} is VP flow matching with residualization and a 5-step Heun ODE, retained as the cheapest well-calibrated operating point. Each row is tuned per dataset under the same fold-0 protocol with an equalized 40-trial budget. The residualizer-C configuration used in the FM row was selected from the prior diagnostic sweep in Appendix~\ref{app:resid-config} and then fixed; it is not tuned jointly with the flow hyperparameters. The frozen score$^+$ bundle---the baseline plus \S\ref{sec:score-side} under a PF-ODE sampler---is subsumed by score-flex and reported as a fast, well-calibrated ablation point in Appendix~\ref{app:competitors}. We also compare against six per-dataset tuned probabilistic baselines: NGBoost~\citep{duan2020ngboost}, quantile-regression LightGBM~\citep{ke2017lightgbm}, CatBoost RMSE-with-uncertainty~\citep{prokhorenkova2018catboost}, iBUG~\citep{brophy2022ibug}, a Gaussian deep ensemble~\citep{lakshminarayanan2017deep}, and CARD-style neural diffusion~\citep{han2022card}.

\begin{table*}[!t]
  \centering
  \caption{Headline real-data results, mean over the eleven benchmark datasets (nine UCI, including the large CT-slice localization dataset, and two scikit-learn) with fold~0 used for tuning and folds~1--5 used for evaluation. rel-CRPS is the per-dataset CRPS divided by the best displayed variant on that dataset. $|\text{cE}|$@$q$ is absolute coverage error at the $q$\% interval level. Wall-clock timing is reported separately over the ten non-CT datasets in Table~\ref{tab:timing}, because the large CT-slice dataset dominates raw-second means (Appendix~\ref{app:large-data}). The CARD row is a compact benchmark adapter, not a fully optimized CARD pipeline (see Appendix~\ref{app:competitors}).}\label{tab:headline}
  \scriptsize
  \begin{tabular}{lrrrrrr}
    \toprule
    Variant & CRPSS $\uparrow$ & rel-CRPS $\downarrow$ & KS $p{>}.05$ $\uparrow$ & $|\text{cE}|$@50 $\downarrow$ & $|\text{cE}|$@90 $\downarrow$ & $|\text{cE}|$@95 $\downarrow$ \\
    \midrule
    Treeffuser-published & 0.699 & 1.248 & 0.27 & 0.142 & 0.047 & 0.026 \\
    DiffGBM-score-flex (ours) & \textbf{0.725} & \textbf{1.106} & 0.33 & 0.112 & 0.049 & 0.037 \\
    DiffGBM-FM (ours) & 0.707 & 1.334 & \textbf{0.45} & \textbf{0.067} & \textbf{0.029} & 0.021 \\
    \midrule
    QReg-LightGBM & 0.694 & 1.826 & 0.31 & 0.085 & 0.045 & 0.043 \\
    CatBoost-unc. & 0.666 & 1.815 & 0.25 & 0.096 & 0.080 & 0.066 \\
    NGBoost & 0.626 & 2.471 & 0.15 & 0.117 & 0.111 & 0.098 \\
    Deep ensemble & 0.685 & 1.556 & 0.24 & 0.108 & 0.036 & \textbf{0.019} \\
    iBUG & 0.679 & 1.564 & 0.33 & 0.119 & 0.052 & 0.060 \\
    CARD-style diffusion & 0.662 & 1.746 & 0.22 & 0.085 & 0.059 & 0.044 \\
    \bottomrule
  \end{tabular}
\end{table*}

\subsection{Headline Result}\label{sec:headline}
Table~\ref{tab:headline} reports cross-dataset means under the tuning/evaluation protocol. The jointly tuned score-flex recipe lifts CRPSS from 0.699 to 0.725 and normalized CRPS from 1.248 to 1.106---the best of any row on both aggregate accuracy metrics---and beats the published baseline on \emph{every} one of the eleven datasets (paired Wilcoxon $11/0$, $W{=}0$, one-sided $p=4.9\times10^{-4}$). A Friedman test over the nine families and eleven datasets rejects rank equality ($\chi^2=32.1$, $p<10^{-3}$), and score-flex holds the best mean CRPS rank ($1.91$) by more than a rank-and-a-half over the next family. FM is the best-calibrated DiffGBM row---lowest 50/90\% coverage error ($0.067$/$0.029$) and highest PIT pass rate ($0.45$)---and the cheapest sampler (Table~\ref{tab:timing}), while remaining competitive on CRPSS ($0.707$); it trades some accuracy on the largest datasets for that speed and calibration. The headline FM row uses the VP path because it is the fastest DiffGBM operating point; the tuned mechanism ablation identifies residualized linear FM as the CRPS-optimal FM corner, so we treat VP and linear as latency and CRPS endpoints rather than a single winning path (Table~\ref{tab:mechanism-ablation}). Per-dataset raw-CRPS winners are a DiffGBM row on $9/11$ datasets (score-flex on seven, FM on two), with only \texttt{diabetes} and \texttt{kin8nm} going to the deep ensemble. Unlike the frozen score$^+$ bundle, score-flex reaches this accuracy with the Euler-SDE sampler rather than the PF-ODE, so it is an accuracy operating point and not a speed win: its coverage error is looser than FM's and its sampling cost tracks the EDM no-residualizer score path (Appendix~\ref{app:large-data}); FM and the frozen score$^+$ ablation (Appendix~\ref{app:competitors}) remain the cheap, tightly calibrated corners. Because the external families are tuned on an equal finite-trial budget but with differing search spaces and per-trial costs (\S\ref{sec:disc}), the cross-family rows are informative about the tradeoff surface rather than definitive head-to-head rankings.

\paragraph{End-to-end timing.} Table~\ref{tab:headline} excludes fitting, but fit time is logged on the same runs, so Table~\ref{tab:timing} reports both over the ten non-CT datasets (the large CT-slice dataset dominates raw-second means and is timed separately in Appendix~\ref{app:large-data}). Two points follow. First, FM is the cheapest end-to-end DiffGBM row (61.3\,s fit$+$sample), well ahead of the published SDE (196.3\,s), so its accuracy and calibration are not bought at latency. Second, the score-flex accuracy row is the most expensive DiffGBM operating point (323.0\,s): its tuned optimum selects the Euler-SDE sampler over an EDM no-residualizer score path, which trades wall-clock for the aggregate CRPS win rather than saving it---the frozen score$^+$ PF-ODE ablation (Appendix~\ref{app:competitors}) is the fast, tightly calibrated score corner. Near-zero sampling time still does not imply lower end-to-end cost: quantile-regression LightGBM samples in 0.30\,s but fits 24--49 separate quantile models, so its total cost is 69.7\,s, and the deep ensemble is 46.3\,s. Fit times are environment-specific and use library-default parallelism, so we read them as the same kind of practical engineering benchmark as the sampling timings.

\begin{table}[!t]
  \centering
  \caption{Fit, sample, and end-to-end time (mean over the ten non-CT benchmarks and evaluation folds; same runs as Table~\ref{tab:headline}). Fit time is per evaluation fold; sample time generates 200 samples per test point.}\label{tab:timing}
  \scriptsize
  \begin{tabular}{lrrr}
    \toprule
    Variant & fit (s) & sample (s) & total (s) \\
    \midrule
    Treeffuser-published & 13.60 & 182.71 & 196.32 \\
    DiffGBM-score-flex (ours) & 26.89 & 296.09 & 322.98 \\
    DiffGBM-FM (ours) & 25.84 & 35.46 & 61.30 \\
    \midrule
    QReg-LightGBM & 69.34 & 0.30 & 69.65 \\
    CatBoost-unc. & 2.34 & \textbf{0.01} & \textbf{2.35} \\
    NGBoost & 7.83 & 0.04 & 7.87 \\
    Deep ensemble & 46.26 & 0.03 & 46.29 \\
    iBUG & \textbf{0.98} & 2.89 & 3.87 \\
    CARD-style diffusion & 23.19 & 42.84 & 66.03 \\
    \bottomrule
  \end{tabular}
\end{table}

\subsection{Per-Dataset CRPSS}\label{sec:perdata}
Table~\ref{tab:per-dataset} shows the per-dataset CRPSS for the two DiffGBM rows and the published baseline. Score-flex is the best DiffGBM row on 8 of the 11 datasets, including all four largest (\texttt{california\_housing}, \texttt{protein}, and both large near-deterministic tasks), while FM wins the three small-to-mid sets where its residualized few-step ODE is strongest (\texttt{yacht}, \texttt{concrete}, \texttt{kin8nm}) and nearly ties on \texttt{energy} (CRPSS equal to three decimals; raw CRPS $0.202$ versus score-flex's $0.199$). The published baseline is no longer the best DiffGBM row on any dataset. The tuned comparison thus reads as a clean split: score-flex is the accuracy choice and dominates as data grows, while FM is the fast, well-calibrated choice on smaller data.

\begin{table}[!t]
  \centering
  \caption{Per-dataset CRPSS (higher is better) and post-split training rows. Bold marks the best DiffGBM row per dataset.}\label{tab:per-dataset}
  \footnotesize
  \begin{tabular}{lrrrr}
    \toprule
    Dataset & Train rows & Published & Score-flex & FM \\
    \midrule
    yacht & 256 & 0.939 & 0.945 & \textbf{0.964} \\
    diabetes & 368 & 0.193 & \textbf{0.241} & 0.215 \\
    energy & 640 & 0.959 & \textbf{0.965} & \textbf{0.965} \\
    concrete & 858 & 0.712 & 0.759 & \textbf{0.772} \\
    wine & 5414 & 0.359 & \textbf{0.395} & 0.310 \\
    kin8nm & 6826 & 0.566 & 0.627 & \textbf{0.631} \\
    power & 7973 & 0.833 & \textbf{0.849} & 0.845 \\
    naval & 9945 & 0.962 & \textbf{0.970} & 0.952 \\
    cali.\ housing & 17200 & 0.682 & \textbf{0.698} & 0.676 \\
    protein & 38108 & 0.499 & \textbf{0.539} & 0.479 \\
    ct slices & 44583 & 0.988 & \textbf{0.989} & 0.966 \\
    \bottomrule
  \end{tabular}
\end{table}

\subsection{Robustness Checks}\label{sec:robustness}
Table~\ref{tab:robustness-ranks} summarizes the same tuned artifacts by per-dataset ranks and dataset-level standard errors. Score-flex takes the best mean CRPS rank ($1.91$) by more than a rank-and-a-half over the next family, with FM second among all nine families; the published baseline drops to fourth. A Friedman test over the nine families and eleven datasets rejects equal CRPS ranks ($\chi^2=32.1$, $p<10^{-3}$), and the Nemenyi critical difference of $3.62$ mean-rank points separates score-flex from the four weakest families (NGBoost, CARD, CatBoost, iBUG). The conservative all-pairs post-hoc does not by itself separate score-flex from the published SDE, quantile-regression LightGBM, or the deep ensemble, but the targeted within-family paired test does: score-flex beats the published baseline on all eleven datasets (Wilcoxon $W{=}0$, $p=4.9\times10^{-4}$; \S\ref{sec:headline}). The aggregate CRPSS gaps remain modest relative to the dataset standard errors, so we read the cross-family ranking as robustness context and the paired within-family test as the decisive claim.

\begin{table}[!t]
  \centering
  \caption{Robustness summary from the tuned evaluation artifacts. Ranks average per-dataset ranks after first averaging folds~1--5; lower ranks are better. SE is the standard error over the eleven dataset-level CRPSS means.}\label{tab:robustness-ranks}
  \scriptsize
  \begin{tabular}{lrrrr}
    \toprule
    Model & CRPS rank & CRPSS $\pm$ SE & rel-CRPS & $|\text{cE}|$@95 rank \\
    \midrule
    DiffGBM-score-flex & \textbf{1.91} & \textbf{0.725} $\pm$ 0.076 & \textbf{1.106} & 4.73 \\
    DiffGBM-FM & 3.45 & 0.707 $\pm$ 0.082 & 1.334 & 3.18 \\
    Treeffuser-published & 4.18 & 0.699 $\pm$ 0.081 & 1.248 & 3.55 \\
    QReg-LightGBM & 4.36 & 0.694 $\pm$ 0.076 & 1.826 & 5.36 \\
    Deep ensemble & 5.45 & 0.685 $\pm$ 0.080 & 1.556 & \textbf{2.73} \\
    iBUG & 5.91 & 0.679 $\pm$ 0.080 & 1.564 & 5.82 \\
    CatBoost-unc. & 6.09 & 0.666 $\pm$ 0.085 & 1.815 & 6.73 \\
    CARD-style diffusion & 6.73 & 0.662 $\pm$ 0.088 & 1.746 & 5.55 \\
    NGBoost & 6.91 & 0.626 $\pm$ 0.089 & 2.471 & 7.36 \\
    \bottomrule
  \end{tabular}
\end{table}

\subsection{Mechanism Ablation}\label{sec:mechanism-ablation}
Table~\ref{tab:mechanism-ablation} isolates the main design choices under the same fold-0 tuning protocol and one shared LightGBM surface with histogram resolution held fixed and identical across rows, so each row varies only its method-defining choice at matched capacity. Three conclusions follow. First, changing only the published noise-prediction score model from Euler SDE sampling to a 25-step Heun probability-flow ODE is not a free speedup: CRPSS falls from 0.670 to 0.629. Second, at this fixed resolution EDM input/target preconditioning is the largest score-side jump, improving CRPSS from 0.646 after residualization alone to 0.680 and cutting 95\% coverage error from 0.027 to 0.016; the flex study confirms the EDM direction under joint binning tuning, where it is selected on ten of the eleven datasets. The log-$\sigma$ feature is roughly neutral under uniform $t$, while the full score$^+$ bundle is the most calibrated score-side row. Third, few-step FM needs residualization: VP-FM without residualization is the weakest rel-CRPS row despite one small-dataset win. Among residualized FM paths, tuned linear has the best aggregate CRPSS, trig is close and wins two datasets, and VP is the fastest. No single path dominates: feature scale, velocity target shape, residualization, and ODE discretization all affect the supervised problem solved by LightGBM.

\begin{table*}[!t]
  \centering
  \caption{Tuned mechanism ablation over the ten non-CT datasets. Each row is tuned on one shared LightGBM surface with histogram resolution (\texttt{max\_bin}) held fixed and identical across rows, so that only the method-defining choice---objective, residualization, probability path, or bound sampler---varies. This is a deliberate ceteris-paribus isolation of each mechanism at matched capacity, complementary to the headline rows, which instead tune histogram resolution jointly for best performance (Table~\ref{tab:headline}). The mechanism \emph{directions} reported here are confirmed to survive joint binning tuning by the score-flex study (e.g.\ the EDM recipe is selected on ten of the eleven datasets); the fixed-resolution magnitudes are read as directional evidence, not capacity-matched effect sizes. rel-CRPS is normalized by the best ablation row on each dataset; the ablation does not include \texttt{ct\_slices}.}\label{tab:mechanism-ablation}
  \scriptsize
  \begin{tabular}{lrrrrrrr}
    \toprule
    Variant & CRPSS $\uparrow$ & rel-CRPS $\downarrow$ & $q$-MACE $\downarrow$ & $|\text{cE}|$@90 $\downarrow$ & $|\text{cE}|$@95 $\downarrow$ & samp.\ time & raw wins \\
    \midrule
    Score noise, Euler-50 & 0.670 & 1.150 & 0.044 & 0.044 & 0.023 & 217.30\,s & 2 \\
    Score noise, Heun-25 & 0.629 & 1.424 & 0.058 & 0.049 & 0.025 & 190.92\,s & 0 \\
    + residualizer & 0.646 & 1.104 & 0.041 & 0.044 & 0.027 & 47.58\,s & 0 \\
    + EDM input/target & 0.680 & 1.042 & 0.034 & 0.025 & 0.016 & 109.46\,s & 0 \\
    + log-$\sigma$ feature & 0.679 & 1.042 & 0.035 & 0.027 & 0.017 & 73.10\,s & 0 \\
    Score$^+$ full & 0.681 & 1.045 & \textbf{0.028} & \textbf{0.020} & \textbf{0.015} & 35.11\,s & 0 \\
    FM linear + residualizer & \textbf{0.684} & \textbf{1.027} & \textbf{0.028} & 0.024 & 0.020 & 19.28\,s & \textbf{4} \\
    FM trig + residualizer & 0.681 & \textbf{1.027} & \textbf{0.028} & 0.026 & 0.021 & 23.87\,s & 2 \\
    FM VP, no residualizer & 0.631 & 1.985 & 0.050 & 0.051 & 0.032 & 23.19\,s & 1 \\
    FM VP + residualizer & 0.680 & 1.034 & 0.034 & 0.028 & 0.024 & \textbf{14.69\,s} & 1 \\
    \bottomrule
  \end{tabular}
\end{table*}

\subsection{Sampler Cost Tradeoff}\label{sec:sampler-cost}
Figure~\ref{fig:sampler-cost} re-evaluates tuned frozen-bundle configurations---published, score$^+$, and FM---across sampler step counts without retuning hyperparameters. As in the mechanism ablation, these configurations come from the earlier fixed-resolution tuning round, so the rows are internally comparable but not identical to the headline configurations of Table~\ref{tab:headline}. The practical pattern is sharp. FM-ODE is effectively converged by 3--5 Heun steps: VP-FM moves from rel-CRPS 1.046 at 3 steps to 1.037 at 5 steps, and residualized linear FM is already at 1.031 at 3 steps. Score$^+$ needs more integration: 5-step PF-ODE is not competitive (rel-CRPS 1.781), while 15--25 steps recover the calibrated score row. The published Euler SDE improves gradually with step count but remains much slower at comparable aggregate CRPS. Finally, the 25-step FM-SDE points confirm the calibration tradeoff: for VP-FM, stochasticity improves 95\% coverage error from 0.024 to 0.018, but worsens $q$-MACE and 90\% coverage error while paying 25 SDE steps. The same tradeoff is clearest in the highest predicted-IQR bin: Appendix~\ref{app:tail-diagnostic} shows VP-FM 95\% coverage improving from 0.906 with the 5-step ODE to 0.944 with the 25-step SDE, and linear FM from 0.902 to 0.935. Thus the default recommendation remains few-step FM-ODE for aggregate CRPS/latency, with stochasticity reserved for conditional-calibration targets.

\begin{figure*}[!t]
  \centering
  \includegraphics[width=\textwidth]{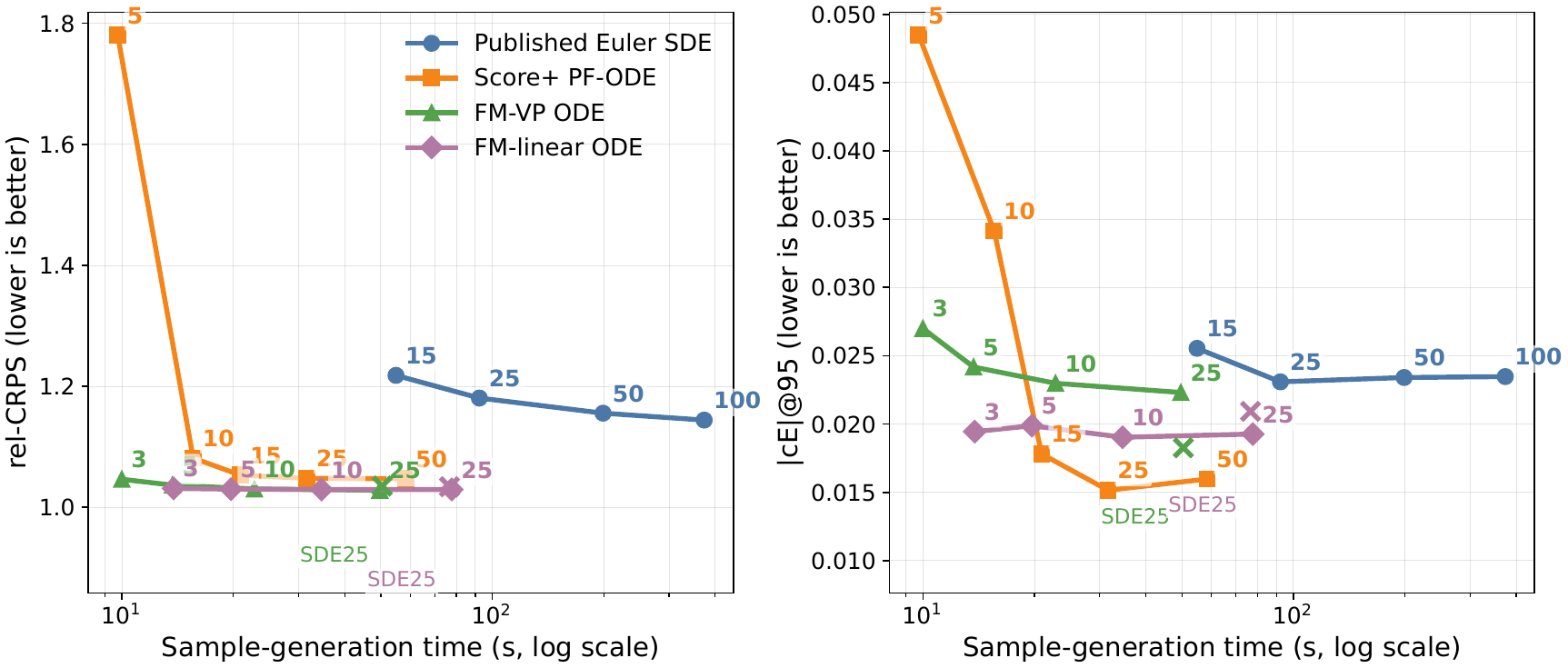}
  \caption{Sampler cost tradeoff after tuning, averaged over the same ten non-CT datasets and folds~1--5. Points are eval-only sampler changes from the tuned configurations; labels show solver steps, and \texttt{SDE25} marks the 25-step stochastic FM sampler. rel-CRPS is normalized by the best sampler-cost row on each dataset.}\label{fig:sampler-cost}
\end{figure*}

\paragraph{Design diagnostics.} The fixed path diagnostic in Appendix~\ref{app:path} isolates the feature-scale advantage of variance-preserving paths under matched residualizer-C and a 5-step Heun ODE; the tuned ablation of Table~\ref{tab:mechanism-ablation} supplies the performance evidence. Stochastic samplers do not Pareto-dominate $\varepsilon=0$ on aggregate CRPS, coverage, or latency, though they can improve high-IQR tail coverage (Appendix~\ref{app:tail-diagnostic}), and transferred score-side tricks are neutral or worse. We therefore recommend the deterministic ODE for aggregate CRPS and latency, and reserve stochasticity as a conditional-calibration knob. Appendix~\ref{app:score-side} separates the tuned headline claims from the fixed diagnostics.

\section{Discussion}\label{sec:disc}
Two findings deserve emphasis. First, the headline rows compare tuned recipes on a shared LightGBM surface at an equalized 40-trial budget, while Appendix~\ref{app:score-side} scopes the fixed diagnostics; because score-flex, FM, and the published baseline share that surface and budget, and score-flex and published additionally share a sampler, the score-flex win is not absorbing a tuning-budget or sampler advantage. Second, per-dataset tuning turns the headline into a clean tradeoff: score-flex is the accuracy row---best aggregate CRPS skill and normalized CRPS, and the winner as data scales---whereas FM is the cheapest sampler and the best-calibrated DiffGBM row. The common mechanism is regression conditioning for tree ensembles: EDM preconditioning stabilises the VE score path's inputs and denoising residual, while FM changes the target and sampler geometry of the supervised problem. The tuned path ablation is important here: variance preservation is one helpful conditioning axis, but not the whole story. Linear FM has a less uniform input scale, yet its simple displacement-like velocity target and direct few-step ODE can be easier to exploit once conditional location is residualized; trig shows that the variance-preserving family remains competitive, but does not uniquely determine the CRPS optimum. Log-$\sigma$ sampling controls where tree capacity is spent, the log-$\sigma$ feature exposes the noise regime to histogram splits, residualization removes predictable conditional location, and Heun ODE sampling is an inference-side efficiency gain. The large-dataset inversion argues against a blanket ``SDE is worse'' interpretation: per-IQR-bin diagnostics show that stochasticity can improve high-uncertainty tail coverage, while aggregate metrics penalise broader intervals, worse CRPS, and higher latency.

\paragraph{Limitations.} All displayed families are tuned per dataset, but this is still finite-budget model selection with unequal search spaces and per-trial costs. The trial budget is equalized within the diffusion families (40 trials each for published, score-flex, and FM, with score-flex and published additionally sharing a sampler), but the six external baselines are tuned for 25 trials each; the cross-family comparison is therefore an equal-within-family finite-trial budget rather than a trial-count- or wall-clock-normalized comparison across families. We include six probabilistic baselines but not distributional random forests~\citep{cevid2022drf}; per-dataset gaps, especially on large datasets, should be read as descriptive rather than high-powered significance estimates. The benchmark suite is also limited to scalar UCI-style tabular regression; categorical-heavy tables, multi-output responses, and operational tail-risk settings remain outside the evidence in this paper. The \texttt{ct\_slices} folds are random over slices rather than grouped by patient (Appendix~\ref{app:datasets}), so its numbers reflect within-patient interpolation; all rows share those folds, so the comparison is internally matched, but rankings under a patient-grouped split are untested. In the FM row the residualizer is selected from fixed diagnostics rather than automated jointly with the flow objective, and it is a partial bottleneck on the largest datasets; the score-flex row instead tunes residualization jointly with the rest of the recipe. The large-data analysis in Appendix~\ref{app:large-data} shows that neither the residualizer nor the sampler alone explains the frozen bundles' behaviour on \texttt{ct\_slices}, and that the jointly tuned score-flex recipe resolves it, beating the published SDE and the deep ensemble there; on that uniquely near-deterministic dataset the selected recipe is noise prediction with finer histogram bins rather than EDM, the reverse of the other ten datasets, so no single recipe axis should be read as universal. On multi-output responses, the implementation fits one LightGBM score or velocity regressor per response coordinate, but each regressor sees the full noisy response vector, so cross-output dependence can in principle be represented through the learned vector field. How well this works with diagonal Gaussian noise paths, per-coordinate tree targets, and finite-step samplers is outside the evidence here and left to future work.

\paragraph{Choosing a variant.} The per-dataset pattern supports two clear operating points. Use \textsc{DiffGBM-score-flex} when aggregate accuracy is the target: its SDE arm has the best CRPS skill of any row, beats the published baseline on every dataset, and dominates as data grows, at the cost of the stochastic sampler's higher latency and looser central coverage. Use \textsc{DiffGBM-FM} when sample latency or interval calibration is the binding constraint; it is the cheapest sampler and the best-calibrated DiffGBM row, most attractive on small and medium data, while the mechanism ablation shows residualized linear FM as the tuned CRPS corner and trig as a close path check. When tight PIT uniformity or interval coverage is the reported quantity and a deterministic sampler is acceptable, the frozen score$^+$ PF-ODE bundle (Appendix~\ref{app:competitors}) and, on the largest near-deterministic datasets, the separately tuned PF-ODE arm of the score-flex space give the strongest coverage. Stochasticity in the FM sampler ($\varepsilon>0$) is a conditional knob, not a default: select it only when high-uncertainty tail coverage is the target (Appendix~\ref{app:tail-diagnostic}), against conditional-calibration metrics rather than aggregate CRPS.

\paragraph{Future work.} Useful extensions include wall-clock-normalized tuning, calibrating both DiffGBM variants and quantile-regression LightGBM under a matched calibration split~\citep{romano2019cqr} and reporting the resulting coverage--width tradeoff, selection criteria for stochastic samplers when conditional coverage matters, categorical-heavy stress tests, multi-output regression, and automatic residualizer selection. A conditional adaptation of Forest-Diffusion~\citep{jolicoeurmartineau2024forest} would also make a natural head-to-head comparison on the regression protocol used here; its released models target joint generation and imputation, so we do not report it as a baseline.

\begin{ack}
This work builds directly on Treeffuser~\citep{treeffuser}, and the DiffGBM implementation inherits its gradient-boosted-tree backbone, SDE module, and feature pipeline under the MIT license. I thank its authors for releasing that code. The released package is available as \texttt{diffgbm} on PyPI; see Appendix~\ref{app:repro}.
\end{ack}

\bibliography{refs}

\appendix

\section{Theoretical Derivations and Design Consequences}\label{app:theory}

\begin{lemma}[Wronskian velocity-to-score identity]\label{lem:wronskian-score}
For a differentiable Gaussian path $y_t=\alpha(t)y_0+\beta(t)z$ with conditional velocity $u_t=\alpha'(t)y_0+\beta'(t)z$, nonzero Wronskian $W(t)=\alpha(t)\beta'(t)-\alpha'(t)\beta(t)$, and $\beta(t)\neq0$ (away from the data endpoint), the population velocity $v(y_t,x,t)=\E[u_t\mid y_t,x,t]$ implies the conditional score
\[
\nabla \log p_t(y_t\mid x)=\frac{\alpha'(t)y_t-\alpha(t)v(y_t,x,t)}{W(t)\beta(t)}.
\]
\end{lemma}
\emph{Proof.} Solving the two linear equations for $z$ gives
\[
z=\frac{\alpha(t)u_t-\alpha'(t)y_t}{\alpha(t)\beta'(t)-\alpha'(t)\beta(t)}
  =\frac{\alpha(t)u_t-\alpha'(t)y_t}{W(t)}.
\]
Because $p(y_t\mid y_0,x)=\Norm(\alpha y_0,\beta^2 I)$, the conditional score is $-z/\beta(t)$. Fisher's identity gives $\nabla\log p_t(y_t\mid x)=\E[-z/\beta(t)\mid y_t,x,t]$. Substituting the expression for $z$ and using that $y_t$ is fixed under the conditional expectation yields the claim. Replacing $v$ by $v_\theta$ gives Eq.~\eqref{eq:score}. \hfill$\square$

The identity is not new theory: it specializes the standard velocity/score relations for Gaussian probability paths and stochastic interpolants~\citep{lipman2023flow,albergo2023stochastic} to the Wronskian form we use, and we present it only as the bridge that lets a velocity-trained tree model drive score-based and stochastic-interpolant samplers without retraining.

At a fixed $t$ with $\alpha(t)\neq0$, the same identity inverts to $v(y_t,x,t)=\alpha'(t)y_t/\alpha(t)-W(t)\beta(t)\nabla\log p_t(y_t\mid x)/\alpha(t)$. Thus score and velocity have the same conditional structure as functions of $y_t$ at that time; the practical differences studied here come from path-dependent scale variation across $t$, feature geometry, and sampler dynamics rather than an intrinsic fixed-$t$ smoothness advantage.

\paragraph{Path-induced feature scale.} If $y_0$ and $z$ are independent with unit variance after standardization, then the path feature satisfies
\[
\Var(y_t)=\alpha(t)^2+\beta(t)^2.
\]
Thus the linear path has $\Var(y_t)=(1-t)^2+t^2$ and compresses to variance $1/2$ at $t=1/2$, whereas the trigonometric and VP paths keep $\Var(y_t)=1$ for all $t$. This identity is elementary; its role here is to connect path choice to the empirical distribution from which histogram split candidates are built.

\begin{proposition}[Pooled histogram-bin proxy]\label{prop:histogram-bins}
Fix a scalar noised-response feature and a finite partition $\mathcal{P}=\{I_1,\ldots,I_B\}$. Let $\mathcal{H}(\mathcal{P})$ be the functions that are constant on each $I_b$, and let
\[
A_t(\mathcal{P})
=
\inf_{g\in\mathcal{H}(\mathcal{P})}
  \E_{p_t}\!\left[(f_t(y_t)-g(y_t))^2\right],
\qquad
f_t(y)=\E[T_t\mid y_t=y,x,t],
\]
where $T_t$ is the scalar score, denoising, noise, or velocity target at time $t$ with covariate value $x$ fixed. Then
\[
A_t(\mathcal{P})
\leq
\sum_{b=1}^B p_t(I_b)\,
  \operatorname{osc}_{I_b}(f_t)^2,
\]
where $\operatorname{osc}_{I_b}(f_t)=\sup_{y,y'\in I_b}|f_t(y)-f_t(y')|$. If $f_t$ is $L_t$-Lipschitz on the bins, then
\[
A_t(\mathcal{P}) \leq L_t^2\sum_{b=1}^B p_t(I_b)|I_b|^2.
\]
Moreover, in a compact-support proxy with smooth positive density $q$ and equal-$q$-mass bins $\mathcal{P}_q$, the high-resolution scaling is
\[
A_t(\mathcal{P}_q)
\lesssim
\frac{L_t^2}{B^2}\int \frac{p_t(y)}{q(y)^2}\,\dd y,
\]
up to lower-order bin-width terms.
\end{proposition}
\emph{Proof.} The best piecewise-constant approximation on a fixed partition cannot be worse than any particular constant choice on each bin. Choosing any value between the infimum and supremum of $f_t$ on $I_b$ gives pointwise squared error at most $\operatorname{osc}_{I_b}(f_t)^2$ on that bin, and integration under $p_t$ gives the first bound. The Lipschitz bound follows from $\operatorname{osc}_{I_b}(f_t)\leq L_t|I_b|$. For the high-resolution proxy, an equal-$q$-mass bin has $|I_b|\approx 1/(B q(\xi_b))$ for some $\xi_b\in I_b$, so
\[
\sum_b p_t(I_b)|I_b|^2
\approx
\sum_b p_t(\xi_b)|I_b|^3
=\frac{1}{B^2}\sum_b \frac{p_t(\xi_b)}{q(\xi_b)^3}\frac{1}{B}
\to \frac{1}{B^2}\int \frac{p_t(y)}{q(y)^2}\,\dd y,
\]
where the last step is a Riemann sum under the measure with density $q$. \hfill$\square$

\paragraph{Shared-bin mismatch.} Proposition~\ref{prop:histogram-bins} is a one-feature proxy for the finite split resolution available to a histogram tree, not a theory of the full boosted LightGBM ensemble. In the benchmark, rows from many noise levels are pooled before tree fitting, so a natural proxy is $q=\bar p$ with
\[
\bar p(y)=\int p(t)\,p_t(y)\,\dd t.
\]
The high-resolution term $\int p_t/\bar p^2$ is large when $\bar p$ allocates little bin density to regions where a particular $p_t$ still places mass. VE paths create exactly this compromise by mixing data-scale low-noise features with expanded high-noise features; a single pooled histogram must spend resolution across both regimes. EDM input preconditioning and VP-FM both reduce this mismatch by keeping the tree feature scale closer to a common range across $t$. Figure~\ref{fig:toy-preconditioning}(b) illustrates the same effect empirically: the published-style VE target is stable, but the unpreconditioned tree input expands across noise levels, whereas EDM and VP-FM keep the feature scale closer to a shared-bin regime. This is a finite-bin conditioning statement, not a claim that the fixed-time VE regression function is intrinsically rough.

\paragraph{Residualization removes predictable location.} Residualization is an exact reparameterization in the infinite-capacity limit: for any fixed $\hat\mu(x)$, modelling $r=y-\hat\mu(x)$ exactly and returning $y=r+\hat\mu(x)$ recovers the same conditional distribution. Its empirical value is finite-model conditioning. In the scalar case, if $\mu(X)=\E[Y\mid X]$ and $R=Y-\mu(X)$, the law of total variance gives
\[
\Var(Y)=\E[\Var(Y\mid X)]+\Var(\mu(X)),
\qquad
\Var(R)=\E[\Var(Y\mid X)].
\]
For vector-valued responses the same statement applies to the trace of the covariance. Thus residualization removes predictable location variation before noised responses are pooled into histogram bins. The large-dataset caveat remains: a fixed residualizer can become a bottleneck once the inner score or velocity model has enough data to represent the conditional mean directly.

\begin{proposition}[Endpoint score-error amplification]\label{prop:endpoint-amp}
Let $v_\theta(y_t,x,t)=v(y_t,x,t)+\delta(y_t,x,t)$ be an approximate velocity with pointwise error $\delta$, and let $\hat s_\theta,s$ denote the implied and true scores from Lemma~\ref{lem:wronskian-score}. Then
\[
\hat s_\theta(y_t,x,t)-s(y_t,x,t)=-\frac{\alpha(t)}{W(t)\,\beta(t)}\,\delta(y_t,x,t),
\]
so the pointwise score-error gain is $|\alpha(t)/(W(t)\beta(t))|$. For the linear and trigonometric paths, $\beta(t)\asymp t$ and $W(t)$ is bounded away from zero near $t=0$, so this gain diverges like $1/t$. For the VP path used here, $\beta(t)\asymp\sqrt{t}$ and $W(t)\asymp1/\sqrt{t}$, so $W(t)\beta(t)$ is bounded away from zero and the velocity-to-score gain remains bounded.
\end{proposition}
\emph{Proof.} Eq.~\eqref{eq:score} is linear in $v$, so
\[
\hat s_\theta - s
=\frac{\alpha'(t)y_t-\alpha(t)(v+\delta)-(\alpha'(t)y_t-\alpha(t)v)}{W(t)\beta(t)}
=-\frac{\alpha(t)\,\delta}{W(t)\beta(t)}. \tag*{$\square$}
\]

\paragraph{Endpoint amplification.} Proposition~\ref{prop:endpoint-amp} therefore predicts that any stochasticity schedule with $\varepsilon(t)\not\to 0$ as $t\to 0$ injects recovered-score error into the SDE drift. In Eq.~\eqref{eq:sde}, the recovered-score contribution to the drift error is
\[
\frac{\varepsilon(t)^2}{2}(\hat s_\theta-s)
=-\frac{\varepsilon(t)^2\alpha(t)}{2W(t)\beta(t)}\,\delta.
\]
More generally, if $\beta(t)\asymp t^a$, $\varepsilon(t)\asymp t^b$, $\alpha(t)\asymp1$, and $W(t)\asymp t^{a-1}$ near the data endpoint, this drift-error contribution scales as $t^{2b+1-2a}\delta$. Linear and trigonometric paths have $a=1$, so constant stochasticity diverges, $\varepsilon(t)\asymp\sqrt{t}$ leaves a non-vanishing endpoint error scale, and $\varepsilon(t)=ct$ makes the recovered-score contribution vanish. The VP path used here has $a=1/2$, so the velocity-to-score gain is milder, but vanishing stochasticity still suppresses the endpoint contribution and $\varepsilon\equiv0$ removes it entirely. This is why $\varepsilon(t)=ct$ and the deterministic corner are the two natural schedules in \S\ref{sec:fm}.

\paragraph{VE population identity.} On a VE path $y_t=y_0+\beta(t)z$ with $\beta'(t)>0$, normalized velocity, noise prediction, and denoising score recovery imply the same population score:
\[
s_t(y_t,x)=-\frac{\E[z\mid y_t,x,t]}{\beta(t)}
          =\frac{\E[y_0\mid y_t,x,t]-y_t}{\beta(t)^2}.
\]
Indeed, for VE, $\alpha=1$, $\alpha'=0$, $u_t=\beta'(t)z$, and $W=\beta'(t)$; Lemma~\ref{lem:wronskian-score} gives $s_t=-\E[z\mid y_t,x,t]/\beta(t)$, and $y_t=y_0+\beta(t)z$ gives the denoising form. This identity is not a claim that the finite LightGBM problems are identical. EDM fits the preconditioned denoising residual $F=(y_0-c_{\mathrm{skip}}y_t)/c_{\mathrm{out}}$ and reconstructs $D_\theta=c_{\mathrm{skip}}y_t+c_{\mathrm{out}}F_\theta$, whereas normalized VE-FM would fit $z$ directly. They imply the same score at the optimum but present different targets and feature scalings to LightGBM. This is the useful interpretation of the score$^+$--FM tie: both improve the conditioning of the regression problem seen by the tree ensemble, EDM by preconditioning the VE score objective and VP-FM by choosing a better-conditioned probability path.

\begin{table}[!h]
  \centering
  \caption{Regression-conditioning view of the main design choices. Here $y_0$ and $z$ are standardized. ``Conditioning issue'' refers to the supervised problem fit by LightGBM, not to the population score identity.}\label{tab:conditioning}
  \scriptsize
  \begin{tabular}{p{0.14\textwidth}p{0.18\textwidth}p{0.19\textwidth}p{0.29\textwidth}}
    \toprule
    Variant & Regressed object & Feature scale & Conditioning issue / fix \\
    \midrule
    VE noise prediction & $-z$ & $y_0+\sigma z$ grows with $\sigma$ & Unit target, but heteroscedastic input and $1/\sigma$ score amplification. \\
    VE EDM score$^+$ & $(y_0-c_{\mathrm{skip}}y_t)/c_{\mathrm{out}}$ & $c_{\mathrm{in}}y_t$ is near unit scale & Input and denoising residual are explicitly preconditioned. \\
    Raw VE-FM & $\beta'(t)z$ & $y_0+\beta z$ grows with $\beta$ & Velocity target inherits the schedule derivative. \\
    Normalized VE-FM & $z$ & $y_0+\beta z$ unless also scaled & Population-equivalent to noise prediction after known rescaling. \\
    VP-FM & $\alpha' y_0+\beta' z$ & $\alpha^2+\beta^2\approx1$ & Path keeps tree features near the standardized data scale. \\
    \bottomrule
  \end{tabular}
\end{table}

\begin{figure}[!h]
  \centering
  \includegraphics[width=\textwidth]{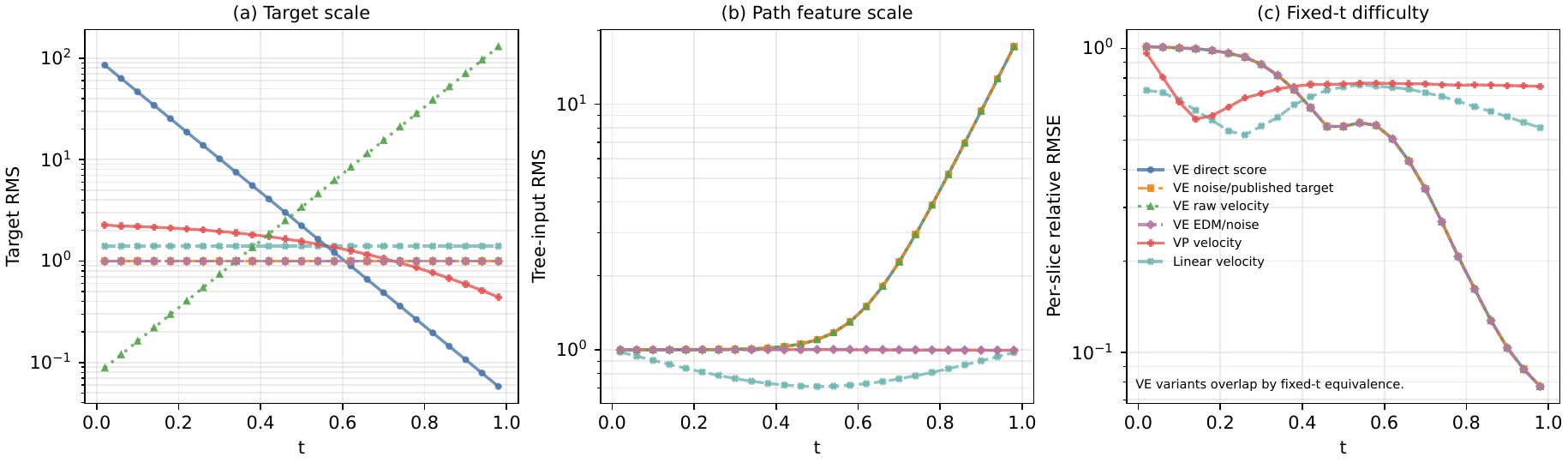}
  \caption{Toy preconditioning diagnostic on a one-dimensional heteroscedastic mixture, generated by \texttt{pixi run toy-geometry}. Panel (a) shows that direct VE score prediction and raw VE-FM have target scales that vary by orders of magnitude across $t$, whereas the published-style VE noise target, EDM/noise prediction, and linear velocity have a target scale that is constant in $t$ (unit scale for the noise targets; $\sqrt{2}$ for the linear velocity $z-y_0$). Panel (b) separates target scaling from input geometry: the published-style VE target is stable but the tree input $y_t=y_0+\sigma z$ still expands strongly, while EDM rescales the input and VP-FM keeps the path feature near standardized scale. Panel (c) fits a fresh LightGBM at each fixed $t$; the VE curves overlap up to numerical noise, confirming that the score/velocity distinction is not intrinsic fixed-$t$ smoothness but cross-$t$ conditioning.}\label{fig:toy-preconditioning}
\end{figure}

\section{Datasets and Protocol Details}\label{app:datasets}

\paragraph{Sources.} Ten of the eleven benchmarks follow the canonical Treeffuser protocol~\citep{treeffuser}: \texttt{yacht}, \texttt{concrete}, \texttt{energy}, \texttt{wine}, \texttt{kin8nm}, \texttt{naval}, \texttt{power\_plant} and \texttt{protein} from the UCI repository, \texttt{california\_housing} from scikit-learn's \texttt{fetch\_california\_housing}, and \texttt{diabetes} from scikit-learn's \texttt{load\_diabetes}. The eleventh, \texttt{ct\_slices}~\citep{graf2011ctslices} (relative location of CT slices on the axial axis; UCI repository, retrieved via OpenML data id 46300; 53{,}500 rows, 384 numeric features), anchors the large-data end of the suite and is analyzed in Appendix~\ref{app:large-data}. Its rows are slices from 74 patients; we drop the patient identifier and assign slices to folds at random, so \texttt{ct\_slices} measures within-patient interpolation rather than generalization to unseen patients. The headline results use the six-fold tuning/evaluation protocol from \S\ref{sec:exp}; Table~\ref{tab:per-dataset} reports the post-split training rows for each evaluation fold. Fixed-configuration appendix diagnostics instead use the train/test sizes declared in their YAML configs, sometimes with large datasets subsampled for sweep cost. The only preprocessing is StandardScaler standardisation of features and target, fit on the training split and applied to the corresponding test split. \texttt{wine} is the union of \texttt{wine-quality-red} and \texttt{wine-quality-white} with a binary \texttt{color} indicator. Although the implementation can use LightGBM's native categorical handling, missing-value routing, and constraint options, the paper experiments do not isolate these features or use categorical-heavy benchmark tables.

\paragraph{Seed policy.} The fold-tuned headline uses the deterministic six-fold manifest with master seed 0. Fixed diagnostic sweeps use the seed lists declared in their configs (usually 0--2 or 0--9), with three independent streams: \texttt{data\_seed} (offset 0) for the train/test split; \texttt{model\_seed} (offset $10{,}000$) for the LightGBM regressors; \texttt{sampler\_seed} (offset $20{,}000$) for stochastic samplers. Variance reduction across seeds therefore reflects only sampling-and-split variability, not boosting initialisation.

\paragraph{Common training settings.} In fixed diagnostic sweeps, unless noted, DiffGBM variants use $n_{\text{estimators}}=3000$, early stopping after 50 rounds, learning rate $0.1$, $n_{\text{repeats}}=30$ (each training point is reused with 30 fresh $t$ samples), and the per-output-dimension LightGBM fitter. In the headline protocol, these and other LightGBM hyperparameters are selected per dataset from the tuning spaces and stored under \texttt{benchmarks/results/tuning/best\_params/}. Inference draws 200 samples per test point.

\paragraph{Tuning budget.} For each dataset, the DiffGBM and published Treeffuser families are tuned for 40 finite Optuna trials on the fold-0 train/validation split---equalized across the published, score-flex, and FM rows so that no diffusion row gets a search advantage over another---with at most twice as many total attempts to absorb failed or non-finite trials. The headline score-flex row is the Euler-SDE arm, tuned in one 40-trial study under the published baseline's sampler; the PF-ODE twin is a separate 40-trial study reported as an additional operating point (Appendix~\ref{app:large-data}). Two of the forty score-flex trials are seeded at the published and score$^+$ corners; the published and FM studies use no seed trials. The tuning objective is validation CRPS computed from 100 predictive Monte Carlo samples per validation point; this is the number of generated response samples used to estimate CRPS, not a subsample of validation observations. Final evaluation refits the selected configuration on each held-out evaluation fold and draws 200 samples per test point. The DiffGBM/Treeffuser rows share the same seven-dimensional LightGBM surface---\texttt{n\_estimators}, \texttt{learning\_rate}, \texttt{num\_leaves}, \texttt{max\_depth}, \texttt{min\_child\_samples}, \texttt{subsample}, and \texttt{max\_bin}---with published and FM fixing their method-defining choices by row and score-flex additionally searching the score-side recipe axes of Table~\ref{tab:tuning-spaces}. The six external baselines are tuned for 25 finite trials each; because their search spaces, per-trial costs, and trial counts differ from the diffusion families, the cross-family comparison should be read as an equal-finite-trial-within-family budget rather than a trial-count- or wall-clock-normalized comparison across families (see Limitations). The persisted YAML files record \texttt{n\_trials\_target}, \texttt{n\_trials\_finite}, and \texttt{n\_trials\_failed}.

\begin{table}[!h]
  \centering
  \caption{Tunable hyperparameter ranges per model family (40 Optuna trials for the DiffGBM/Treeffuser families, 25 for the external baselines). Log = log-uniform sampling. Fixed method-defining parameters are not listed.}\label{tab:tuning-spaces}
  \scriptsize
  \begin{tabular}{llll}
    \toprule
    Family & Tunable parameter & Range & Scale \\
    \midrule
    Shared LGBM surface     & \texttt{n\_estimators}     & 200--3000 & log \\
    (all DiffGBM/Treeffuser & \texttt{learning\_rate}    & 0.001--0.3 & log \\
    rows)                   & \texttt{num\_leaves}       & 15--255 & linear \\
                            & \texttt{max\_depth}        & $\{-1, 4, 6, 8, 10\}$ & categorical \\
                            & \texttt{min\_child\_samples} & 5--100 & linear \\
                            & \texttt{subsample}         & 0.5--1.0 & linear \\
                            & \texttt{max\_bin}          & $\{255, 1023, 4095\}$ & categorical \\
    \midrule
    DiffGBM-score-flex      & \texttt{score\_parameterization} & $\{$noise, edm$\}$ & categorical \\
    (adds, on top of        & \texttt{noise\_features}   & $\{$raw\_time, raw\_time\_log\_std$\}$ & categorical \\
    the shared surface;      & \texttt{t\_sampling}       & $\{$uniform, log\_sigma\_normal$\}$ & categorical \\
    headline row uses the   & \texttt{log\_sigma} prior $\mu$ & $-3$--$0$ & linear \\
                            & \texttt{log\_sigma} prior $\sigma$ & $0.6$--$2.0$ & linear \\
    Euler-SDE sampler)      & \texttt{loss\_weighting}   & $\{$uniform, min\_snr$\}$ & categorical \\
                            & \texttt{min\_snr\_gamma}   & 1--5 & linear \\
                            & \texttt{residualize}       & $\{$off, mean$\}$ ($+$capacity) & categorical \\
    \midrule
    NGBoost     & \texttt{n\_estimators}     & 200--3000 & log \\
                & \texttt{learning\_rate}    & 0.001--0.3 & log \\
    \midrule
    iBUG        & \texttt{k}                 & 10--200 & log \\
                & \texttt{n\_estimators}     & 200--2000 & log \\
                & \texttt{learning\_rate}    & 0.001--0.3 & log \\
                & \texttt{max\_depth}        & 3--10 & linear \\
    \midrule
    QReg-LightGBM & \texttt{quantile\_count} & 24--49 & linear \\
                   & \texttt{n\_estimators}   & 50--500 & log \\
                   & \texttt{learning\_rate}  & 0.02--0.2 & log \\
                   & \texttt{num\_leaves}     & 15--127 & linear \\
    \midrule
    CatBoost    & \texttt{iterations}        & 200--3000 & log \\
                & \texttt{learning\_rate}    & 0.001--0.3 & log \\
                & \texttt{depth}             & 4--10 & linear \\
    \midrule
    Deep ensemble & \texttt{hidden\_size}    & $\{64, 128, 256, 512\}$ & categorical \\
                  & \texttt{n\_layers}       & 2--5 & linear \\
                  & \texttt{learning\_rate}  & 0.0001--0.01 & log \\
                  & \texttt{max\_epochs}     & $\{100, 200, 400\}$ & categorical \\
    \midrule
    CARD-style  & \texttt{hidden\_size}      & $\{64, 128, 256\}$ & categorical \\
                & \texttt{n\_layers}         & 2--5 & linear \\
                & \texttt{learning\_rate}    & 0.0001--0.01 & log \\
                & \texttt{max\_epochs}       & $\{200, 400\}$ & categorical \\
                & \texttt{diffusion\_epochs} & $\{200, 400\}$ & categorical \\
                & \texttt{n\_steps}          & $\{50, 100, 200\}$ & categorical \\
    \bottomrule
  \end{tabular}
\end{table}

\paragraph{Appendix calibration metrics.} Appendix tables use $q$-MACE for sample-rank calibration: for each held-out response, we compute the fraction of predictive samples above the observed value, sort these ranks, and average their absolute deviation from an evenly spaced uniform grid. Lower $q$-MACE therefore means the sample predictive distributions have ranks closer to uniform, analogous to the PIT-KS metric in the main table but reported as an error magnitude.

\paragraph{Sampler details.}
\textsc{Treeffuser-published} uses the VESDE reverse-time SDE with the Euler--Maruyama sampler at $n_{\text{steps}}=50$, matching the published code path. \textsc{Treeffuser-score$^+$} uses the Heun probability-flow ODE at $n_{\text{steps}}=25$ steps under the EDM-style parametrization with $\sigma_{\min}=0.01$, $\sigma_{\max}=20$. \textsc{DiffGBM-FM} uses the Heun ODE at $n_{\text{steps}}=5$ on the VP path with $\beta_{\min}=0.1$, $\beta_{\max}=20$.

\section{Claim Scope and Diagnostic Cross-References}\label{app:score-side}

Unless otherwise stated, appendix sweeps are fixed-configuration diagnostics used to isolate modelling choices; they are not per-dataset tuned headline comparisons. The experimental pipeline is: fixed diagnostics identify plausible design choices; those choices become the tunable recipe axes of the headline \textsc{DiffGBM-score-flex} space and define the frozen \textsc{published}, \textsc{score$^+$}, and \textsc{FM} reference families; the headline space and the frozen families are then tuned per dataset under the fold-0 protocol at an equalized 40-trial budget, with the headline score-flex row fixed to the published baseline's Euler-SDE sampler; and only that tuned protocol supports the headline aggregate performance claims.

\begin{table}[!h]
  \centering
  \caption{Scope of the empirical claims. ``Tuned'' means per-dataset Optuna selection on fold~0 before evaluation on folds~1--5; ``fixed'' means shared diagnostic settings declared in the corresponding YAML configs.}\label{tab:claim-map}
  \scriptsize
  \begin{tabular}{@{}p{0.22\textwidth}p{0.24\textwidth}p{0.16\textwidth}p{0.27\textwidth}@{}}
    \toprule
    Claim & Evidence & Tuning status & Role in the paper \\
    \midrule
    \textsc{DiffGBM-score-flex} beats the published baseline on every dataset and \textsc{FM} improves the tradeoff surface over it & Table~\ref{tab:headline}, Table~\ref{tab:per-dataset}, \S\ref{sec:headline} & Tuned per dataset & Headline performance claim (paired Wilcoxon $11/0$) \\
    \textsc{DiffGBM-FM} is the most calibrated headline row; the PF-ODE score corners give the tightest interval coverage & Table~\ref{tab:headline}, Table~\ref{tab:ct-slices-tuned} & Tuned per dataset & Headline aggregate calibration claim \\
    VP-FM is the preferred FM path among the tested paths & Table~\ref{tab:path} & Fixed diagnostic & Mechanism and model-selection evidence; not a claim that VP would win every independently tuned path comparison \\
    Deterministic FM is the recommended aggregate CRPS/latency sampler & Table~\ref{tab:stoch-eps} with Table~\ref{tab:tail-diagnostic} & Fixed diagnostic & Sampler diagnostic; stochasticity remains a conditional-calibration knob \\
    Residualizer-C is a reasonable headline residualizer choice & Tables~\ref{tab:off-vs-mean} and~\ref{tab:resid-AE} & Fixed diagnostic & Design choice for the frozen model families; automatic residualizer selection is left as future work \\
    Log-$\sigma$ sampling, explicit noise-level features, and EDM preconditioning motivate the score-side recipe axes & Toy diagnostic and development sweeps listed in Table~\ref{tab:appendix-provenance} & Fixed diagnostic & Mechanism evidence for the axes; the quantitative performance claim is the tuned \textsc{DiffGBM-score-flex} row \\
    \bottomrule
  \end{tabular}
\end{table}

The four score-side modifications of \S\ref{sec:score-side} were validated by smoke and synthetic-core sweeps before being exposed as tunable axes of the score-flex space (and combined in the frozen \textsc{score$^+$} reference bundle).
The residualizer-C configuration that the FM row uses was selected by the residualizer sweep summarised in Appendix~\ref{app:resid-config}. Log-sigma time sampling, the log-sigma noise feature, and EDM preconditioning were locked in as candidate axes by the \texttt{log\_sigma\_sweep}, \texttt{synthetic\_core}, and \texttt{pf\_ode\_sweep} development runs. We treat those runs as model-selection evidence rather than as separately reported tuned comparisons; the quantitative claim in the paper is the tuned \textsc{DiffGBM-score-flex} row in Table~\ref{tab:headline}, where these axes are searched jointly on the full eleven-dataset protocol.

\section{Residualizer Configuration}\label{app:resid-config}

\subsection{Off vs.\ Mean Residualization (FM)}\label{app:off-vs-mean}

We isolate the contribution of mean residualization on the FM side at matched configuration (VP path, uniform $t$, uniform loss weighting, residualizer-C extra parameters) using a ten-dataset mixed synthetic/real diagnostic suite and ten seeds.

\begin{table}[!h]
  \centering
  \caption{FM with mean residualization vs.\ no residualization (\texttt{off}). Paired differences (mean $-$ off, 100 paired observations); negative is better for all metrics shown.}\label{tab:off-vs-mean}
  \footnotesize
  \begin{tabular}{lrrrrr}
    \toprule
    Sampler & $\Delta$CRPS & $\Delta q$-MACE & $\Delta$KS stat & $\Delta\,|\text{cE}|$@90 & paired $t$ on $\Delta$CRPS \\
    \midrule
    ODE @ 5 steps  & $-0.23$ & $-0.022$ & $-0.034$ & $-0.032$ & $-4.30$ \\
    SDE @ 25 steps & $-0.05$ & $+0.001$ & $+0.003$ & $-0.001$ & $-2.91$ \\
    \bottomrule
  \end{tabular}
\end{table}

The ODE-5 advantage is large and broadly significant across calibration metrics (81/100 paired wins on $q$-MACE and 81/100 on $|\text{cE}|$@90). At SDE-25 the gap on calibration metrics is null (paired $t<1$); the aggregate CRPS gain is driven by two real-data outliers (\texttt{yacht} $+54\%$ relative, \texttt{concrete} $+15\%$) while six of ten datasets actually favour the unresidualized baseline. The mechanism is consistent with the few-step ODE being unable to absorb the conditional mean itself when the velocity net has $\leq 5$ Heun updates; with 25 SDE updates the velocity model recovers calibration on its own. Cost: mean residualization adds $\approx 12$\,s of fit time on these sizes ($\sim 8\times$ vs.\ no residualization), so the trade is most attractive in low-step regimes.

\subsection{Residualizer Variants A--E (FM side)}\label{app:resid-AE}

Five residualizer configurations were swept on a four-dataset, three-seed sub-suite (\texttt{california\_housing}, \texttt{diabetes}, \texttt{energy}, \texttt{protein}-5k subsample) under VP-FM-ODE at 5 steps; see Table~\ref{tab:resid-AE}.

\begin{table}[!h]
  \centering
  \caption{FM residualizer diagnostic. ``Conf.\ leaves'' is the LightGBM \texttt{num\_leaves}; ``$n_{\text{est}}^{(r)}$'' is the residualizer's \texttt{n\_estimators}; ``ES'' is early-stopping rounds in the residualizer (off means no ES).}\label{tab:resid-AE}
  \footnotesize
  \begin{tabular}{llrrrrrrr}
    \toprule
    Variant & Description & Leaves & $n_{\text{est}}^{(r)}$ & ES & CRPSS $\uparrow$ & $q$-MACE $\downarrow$ & $|\text{cE}|$@95 $\downarrow$ & fit s \\
    \midrule
    A & current-baseline   & 31 & 100 & off & 0.521 & 0.028 & 0.017 & 1.82 \\
    B & regularised        & 31 & 100 & 30  & 0.509 & 0.042 & 0.018 & 2.78 \\
    C & high capacity      & 63 & 300 & off & \textbf{0.523} & 0.039 & 0.011 & 8.03 \\
    D & ES moderate        & 63 & 500 & 30  & 0.516 & 0.033 & 0.021 & 5.91 \\
    E & ES + high capacity & 63 & 2000& 30  & 0.522 & 0.034 & \textbf{0.010} & 9.81 \\
    \bottomrule
  \end{tabular}
\end{table}

Within this fixed diagnostic, configurations C and E sit on the calibration--latency Pareto front. C is the headline choice (best CRPSS, $1.2\times$ faster than E); E edges C on tail coverage ($|\text{cE}|$@95 $0.010$ vs.\ $0.011$). A separate full-protein diagnostic shows E reclaiming the lead on \texttt{protein} (CRPSS $0.476$ vs.\ C's $0.449$), reinforcing the case that the optimal residualizer is dataset-dependent.

\subsection{Scale-Aware Residualization}\label{app:mean-scale}

The \texttt{mean\_scale} residualizer additionally divides residuals by a cross-validated conditional-std estimate $\hat\sigma(x)$. On the ten-dataset mixed synthetic/real diagnostic suite with ten seeds:

\begin{table}[!h]
  \centering
  \caption{Mean vs.\ mean-scale residualization on FM at matched configuration. Higher CRPSS is better; lower CRPS, DSS, $q$-MACE are better.}\label{tab:mean-scale}
  \footnotesize
  \begin{tabular}{lrrrr}
    \toprule
    Variant & CRPSS $\uparrow$ & CRPS $\downarrow$ & DSS $\downarrow$ & $q$-MACE $\downarrow$ \\
    \midrule
    VP-FM-ODE, mean       & \textbf{0.519} & \textbf{0.652} & \textbf{$+0.314$} & \textbf{0.031} \\
    VP-FM-ODE, mean-scale & 0.489 & 0.724 & $+0.665$ & 0.038 \\
    VP-FM-SDE, mean       & \textbf{0.517} & \textbf{0.657} & \textbf{$+0.380$} & \textbf{0.040} \\
    VP-FM-SDE, mean-scale & 0.475 & 0.750 & $+0.676$ & 0.049 \\
    \bottomrule
  \end{tabular}
\end{table}

Scale-aware residualization is uniformly worse: $+11\%$ CRPS on the ODE path, $+14\%$ on the SDE path, with DSS more than doubling. The per-$x$ scale estimator is itself noisy in these finite-sample regimes; dividing by it injects more variance into the velocity target than it removes. \texttt{mean} is therefore the only residualizer carried into the headline.

\section{Time Sampling and Loss Weighting}\label{app:t-and-loss}

\subsection{Time Sampling}\label{app:t-sampling}

\begin{table}[!h]
  \centering
  \caption{FM time-sampling sweep on fixed train/test splits on the ten non-CT benchmarks, 3 seeds each, matched residualizer-C, VP path. ``uniform'' is the headline choice with an endpoint anchor at $t=1$ at probability $0.05$.}\label{tab:t-sampling}
  \footnotesize
  \begin{tabular}{lrrrrr}
    \toprule
    Variant & CRPSS $\uparrow$ & DSS $\downarrow$ & $q$-MACE $\downarrow$ & KS $p{>}.05$ $\uparrow$ & samp s \\
    \midrule
    VP-FM-ODE uniform           & \textbf{0.6564} & $-0.154$ & 0.040 & 0.47 & 3.63 \\
    VP-FM-ODE uniform anchor16  & 0.6560 & \textbf{$-0.171$} & 0.040 & 0.53 & 3.84 \\
    VP-FM-ODE logbeta           & 0.6556 & $-0.175$ & \textbf{0.035} & \textbf{0.77} & 3.74 \\
    VP-FM-ODE logSNR            & 0.6538 & $-0.135$ & 0.042 & 0.40 & 3.71 \\
    \midrule
    score$^+$ uniform-$t$       & 0.6556 & \textbf{$-0.210$} & 0.038 & 0.50 & 6.95 \\
    score$^+$ log-$\sigma$ $t$  & \textbf{0.6565} & $-0.169$ & 0.038 & \textbf{0.70} & 6.58 \\
    \bottomrule
  \end{tabular}
\end{table}

For FM, log-beta and log-SNR sampling do not move CRPS by more than $0.4\%$ relative to uniform with endpoint anchor; the smoother density shape gains $+5\%$ on KS pass rate and $-13\%$ on $q$-MACE but the effect is within seed variance for CRPSS. The headline keeps uniform-$t$ for parsimony; log-beta is a viable substitute when KS uniformity is the priority. For the score model, log-sigma sampling clearly helps PIT uniformity (KS pass $0.50 \to 0.70$) at the cost of slightly worse DSS; earlier score-side sweeps found the EDM-default $\Norm(-1.2,1.2^2)$ setting to be the safest dataset-agnostic choice among the tested log-$\sigma$ distributions. We keep log-sigma in the \textsc{score$^+$} headline because the KS gain is the more interpretable calibration win.

\subsection{\texorpdfstring{Loss Weighting (min-SNR-$\gamma$)}{Loss Weighting (min-SNR-gamma)}}\label{app:loss-weighting}

Min-SNR-$\gamma$~\citep{hang2023minsnr} caps the per-noise-level loss weight at $\min(\mathrm{SNR}(t),\gamma)$, limiting the influence of the low-noise rows whose targets carry the largest gradients. It is one of the tunable loss-weighting axes of the score-flex space (Table~\ref{tab:tuning-spaces}); this fixed diagnostic isolates its effect on the FM side.

\begin{table}[!h]
  \centering
  \caption{Min-SNR-$\gamma$ loss weighting on FM, fixed train/test splits on the ten non-CT benchmarks, 3 seeds each. $\gamma=5$ collapses to uniform on most of the support; $\gamma=1$ amplifies the late-time loss most aggressively.}\label{tab:minsnr}
  \footnotesize
  \begin{tabular}{lrrrrr}
    \toprule
    Variant & CRPSS $\uparrow$ & CRPS $\downarrow$ & DSS $\downarrow$ & $q$-MACE $\downarrow$ & KS $p{>}.05$ $\uparrow$ \\
    \midrule
    VP-FM-ODE uniform-$w$ & \textbf{0.6564} & \textbf{4.097} & $-0.159$ & 0.040 & 0.47 \\
    VP-FM-ODE min-SNR-5   & 0.6564 & 4.119 & $-0.167$ & \textbf{0.037} & \textbf{0.60} \\
    VP-FM-ODE min-SNR-1   & 0.6516 & 4.223 & $-0.104$ & 0.044 & 0.37 \\
    VP-FM-SDE uniform-$w$ & \textbf{0.6540} & \textbf{4.114} & $-0.080$ & 0.051 & 0.30 \\
    VP-FM-SDE min-SNR-5   & 0.6538 & 4.138 & $-0.085$ & 0.050 & 0.30 \\
    VP-FM-SDE min-SNR-1   & 0.6485 & 4.239 & $-0.046$ & 0.057 & 0.23 \\
    \bottomrule
  \end{tabular}
\end{table}

Min-SNR-$\gamma$ is a sharpness-versus-tail trade-off. $\gamma=1$ hurts CRPS by $0.8$--$1.3\%$ across both samplers (Wilcoxon $p \leq 0.003$ for both); $\gamma=5$ is statistically indistinguishable from uniform weighting on CRPS, but improves PIT KS pass rate on the ODE path. We keep uniform weighting in the FM headline.

\section{Stochasticity (Velocity-Noise) Sweeps}\label{app:stoch}

\subsection{Stochasticity Strength}\label{app:stoch-eps}

We sweep velocity stochasticity on the VP path under linear and $\sqrt{t}$ schedules with residualizer-C, plus one residualizer-E check; results in Table~\ref{tab:stoch-eps} are averages over fixed train/test splits on the ten non-CT benchmarks, 3 seeds each, at 25 SDE steps.

\begin{table}[!h]
  \centering
  \caption{Stochasticity sweep. The $\varepsilon=0$ corner is the deterministic FM recipe (Heun ODE @ 5 steps); SDE rows use 25 steps. Sample time is omitted because the ODE and SDE rows use different benchmark protocols.}\label{tab:stoch-eps}
  \footnotesize
  \begin{tabular}{lrrrr}
    \toprule
    Variant & CRPSS $\uparrow$ & CRPS $\downarrow$ & $|\text{cE}|$@90 $\downarrow$ & $|\text{cE}|$@95 $\downarrow$ \\
    \midrule
    $\varepsilon=0$ headline (ODE)    & 0.669 & --     & \textbf{0.024} & \textbf{0.017} \\
    $\varepsilon=0.25$ linear, res-C  & 0.656 & 4.099  & 0.028 & 0.015 \\
    $\varepsilon=0.50$ linear, res-C  & 0.655 & 4.104  & 0.036 & 0.017 \\
    $\varepsilon=1.00$ linear, res-C  & 0.654 & 4.114  & 0.047 & 0.022 \\
    $\varepsilon=0.50$ sqrt, res-C    & 0.655 & 4.101  & 0.035 & 0.017 \\
    $\varepsilon=1.00$ sqrt, res-C    & 0.655 & 4.107  & 0.043 & 0.020 \\
    $\varepsilon=1.00$ linear, res-E  & 0.651 & 4.128  & 0.048 & 0.024 \\
    \bottomrule
  \end{tabular}
\end{table}

In this fixed diagnostic sweep, no tested SDE setting Pareto-dominates $\varepsilon=0$. The closest contender ($\varepsilon=0.25$ linear) matches $|\text{cE}|$@95 ($0.015$ vs.\ $0.017$) but is strictly worse on CRPSS and $|\text{cE}|$@90, and pays $25$ SDE steps vs.\ $5$ ODE steps.

\subsection{Schedule Shape}\label{app:stoch-schedule}

The schedule $\varepsilon(t) = c\,t$ (linear) is the headline form. The $\sqrt{t}$ alternative was tested for parity at $\varepsilon \in \{0.5, 1.0\}$ and produces statistically indistinguishable CRPSS (Table~\ref{tab:stoch-eps}). Both schedules vanish at the data endpoint, which suppresses recovered-score error in the stochastic drift (Eq.~\eqref{eq:score} and Proposition~\ref{prop:endpoint-amp}); neither shape helps once $\varepsilon$ is small.

\section{Fixed Path Diagnostic}\label{app:path}

\begin{table}[!h]
  \centering
  \caption{Flow-path diagnostic under the 5-step deterministic ODE sampler: linear vs.\ trigonometric vs.\ variance-preserving, 8 datasets $\times$ 3 seeds, matched residualizer-C. Lower is better for all columns; bold is best per column.}\label{tab:path}
  \footnotesize
  \begin{tabular}{lrrrr}
    \toprule
    Path & CRPS $\downarrow$ & $q$-MACE $\downarrow$ & $|\text{cE}|$@95 $\downarrow$ & samp s \\
    \midrule
    Linear         & 4.324 & \textbf{0.026} & 0.032 & 0.49 \\
    Trigonometric  & 4.306 & 0.028 & \textbf{0.028} & \textbf{0.43} \\
    Variance-preserving (VP) & \textbf{4.294} & 0.028 & \textbf{0.028} & 0.46 \\
    \bottomrule
  \end{tabular}
\end{table}

This fixed diagnostic is meant to isolate geometry, not to select the final path. Under matched residualizer-C, a shared 5-step ODE sampler, and fixed LightGBM settings, VP is best on CRPS, while trig and VP both reduce 95\% coverage error relative to the linear path. The pattern tracks variance preservation: VP and trig satisfy $\alpha^2+\beta^2 \approx 1$, keeping $y_t$ at roughly unit scale across $t$ and stabilising the tree regressor's feature space. The tuned mechanism ablation in Table~\ref{tab:mechanism-ablation} supersedes this fixed ranking for performance claims: with per-dataset LightGBM tuning, linear has the best aggregate CRPSS, trig is competitive and wins two datasets, and VP remains fastest. We therefore read path choice as a finite-model tradeoff: variance-preserving paths stabilise feature scale, while linear FM can have a simpler displacement-like velocity target that wins CRPS after residualization and tuning.

\section{Probabilistic-Family Baselines}\label{app:competitors}

We compare the headline DiffGBM rows and the published baseline against six external probabilistic regressors (Table~\ref{tab:competitors}). Hyperparameters for every displayed family are selected per dataset on the same fold-0 tuning split, then evaluated on folds~1--5. The CARD row uses a compact benchmark adapter with the same two-stage conditional-mean plus conditional-diffusion structure as CARD~\citep{han2022card}, so conclusions involving CARD should be read as comparisons to this adapter rather than to a fully optimized CARD pipeline.

\begin{table}[!h]
  \centering
  \caption{Probabilistic baselines under the tuning/evaluation protocol, mean over the eleven datasets. rel-CRPS is normalized by the best displayed variant on each dataset. \textsc{Treeffuser-score$^+$} is the frozen PF-ODE score bundle, shown as the calibration ablation (it is subsumed by score-flex). Sample-time excludes fitting and is dominated by the large CT-slice dataset (see Table~\ref{tab:timing} for a ten-non-CT breakdown). The iBUG DSS is dominated by a single \texttt{wine} fold where the model assigns near-zero predictive variance to a mispredicted test point.}\label{tab:competitors}
  \scriptsize
  \begin{tabular}{lrrrrrrr}
    \toprule
    Model & CRPSS $\uparrow$ & rel-CRPS $\downarrow$ & CRPS $\downarrow$ & DSS $\downarrow$ & $q$-MACE $\downarrow$ & $|\text{cE}|$@95 $\downarrow$ & samp s \\
    \midrule
    DiffGBM-score-flex (ours)       & \textbf{0.725} & \textbf{1.108} & 3.687 & $-$0.093 & 0.045 & 0.037 & 508.0 \\
    Treeffuser-score$^+$ (ablation) & 0.708 & 1.372 & 3.764 & $-$0.166 & \textbf{0.027} & \textbf{0.018} & 421.8 \\
    DiffGBM-FM (ours)               & 0.707 & 1.335 & 3.829 & \textbf{$-$0.178} & 0.036 & 0.021 & 60.8 \\
    Treeffuser-published            & 0.699 & 1.249 & 3.953 & 1.056 & 0.053 & 0.026 & 759.9 \\
    Quantile-regression LightGBM    & 0.694 & 1.827 & 3.721 & 0.193 & 0.038 & 0.043 & 0.4 \\
    Deep ensemble                   & 0.685 & 1.557 & \textbf{3.658} & $-$0.049 & 0.054 & 0.019 & 0.1 \\
    iBUG (XGBoost)                  & 0.679 & 1.566 & 3.833 & $1.68{\times}10^7$ & 0.069 & 0.060 & 5.0 \\
    CatBoost (uncertainty)          & 0.666 & 1.816 & 3.743 & 0.367 & 0.045 & 0.066 & \textbf{0.0} \\
    CARD-style diffusion            & 0.662 & 1.747 & 4.426 & 0.435 & 0.050 & 0.044 & 62.1 \\
    NGBoost (Gaussian)              & 0.626 & 2.472 & 4.048 & 2.035 & 0.051 & 0.098 & 2.8 \\
    \bottomrule
  \end{tabular}
\end{table}

Under this per-dataset tuning protocol, the DiffGBM rows lead on aggregate CRPSS and rel-CRPS, with score-flex best on both; the frozen score$^+$ ablation is the strongest interval-calibration row ($q$-MACE $0.027$, $|\text{cE}|$@95 $0.018$), which is the role it plays now that score-flex is the headline accuracy row. The external baselines remain important: deep ensembles have the best mean raw CRPS, though this scale-sensitive average is less informative about cross-dataset dominance than CRPSS or rel-CRPS, and among external families only the deep ensemble takes a per-dataset raw-CRPS win (\texttt{diabetes} and \texttt{kin8nm}); every other dataset goes to a DiffGBM row. Their aggregate rel-CRPS remains worse, indicating that no single tuned external family transfers as evenly across the benchmark suite as DiffGBM.

\section{Extended Per-Dataset Results}\label{app:perdata-ext}

Table~\ref{tab:perdata-ext} reports CRPS, $|\text{cE}|$@90, and sample-generation time on each of the eleven benchmarks for the published baseline and the two DiffGBM rows (Flex $=$ score-flex SDE), averaged over the five evaluation folds. The ``best margin'' column gives the gap between the winning row and the second-best in raw CRPS. The per-row picture matches the headline: among the three DiffGBM rows, score-flex wins raw CRPS on eight datasets---by growing margins as data scales (up to $+11\%$ on \texttt{ct\_slices})---while FM wins the three small-to-mid sets where its residualized ODE and near-zero sample time are most valuable, and the published baseline wins none.

\begin{table}[!h]
  \centering
  \caption{Extended per-dataset DiffGBM results (with the published baseline), ordered roughly by evaluation-fold training size. ``best margin'' is winner-over-second-best in raw CRPS; the bold cell in each block is the per-metric winner on that row among the three rows shown, not against the external baselines (the deep ensemble takes the raw-CRPS win on \texttt{diabetes} and \texttt{kin8nm}; \S\ref{sec:headline}). Numbers come from \texttt{benchmarks/results/tuning/eval/*.jsonl}.}\label{tab:perdata-ext}
  \scriptsize
  \begin{tabular}{lrrrrrrrrrr}
    \toprule
    & & \multicolumn{3}{c}{CRPS $\downarrow$} & \multicolumn{3}{c}{$|\text{cE}|$@90 $\downarrow$} & \multicolumn{3}{c}{sample s} \\
    \cmidrule(lr){3-5} \cmidrule(lr){6-8} \cmidrule(lr){9-11}
    Dataset & best margin & Pub.\ & Flex & FM & Pub.\ & Flex & FM & Pub.\ & Flex & FM \\
    \midrule
    yacht            & FM\,$+$50.0\% & 0.447 & 0.404 & \textbf{0.269} & 0.084 & 0.055 & \textbf{0.046} & 0.6 & 2.1 & \textbf{0.1} \\
    diabetes         & Flex\,$+$3.6\% & 36.04 & \textbf{33.98} & 35.19 & 0.094 & 0.086 & \textbf{0.071} & 0.7 & 0.2 & \textbf{0.0} \\
    energy           & Flex\,$+$1.5\% & 0.236 & \textbf{0.199} & 0.202 & 0.077 & 0.030 & \textbf{0.022} & 7.3 & 7.2 & \textbf{0.1} \\
    concrete         & FM\,$+$5.5\% & 2.71 & 2.27 & \textbf{2.15} & 0.034 & 0.021 & \textbf{0.021} & 1.7 & 7.1 & \textbf{0.1} \\
    wine             & Flex\,$+$5.9\% & 0.294 & \textbf{0.278} & 0.317 & 0.033 & 0.037 & \textbf{0.018} & 35.5 & 151.5 & \textbf{15.5} \\
    kin8nm           & FM\,$+$1.1\% & 0.065 & 0.056 & \textbf{0.055} & 0.012 & \textbf{0.008} & 0.021 & 144.8 & \textbf{34.7} & 35.8 \\
    power            & Flex\,$+$3.2\% & 1.63 & \textbf{1.47} & 1.52 & \textbf{0.008} & 0.090 & 0.010 & 155.1 & 433.9 & \textbf{5.2} \\
    naval            & Flex\,$+$26.5\% & 0.000163 & \textbf{0.000129} & 0.000209 & 0.087 & \textbf{0.057} & 0.060 & 460.5 & 533.7 & \textbf{60.7} \\
    cali.\ housing   & Flex\,$+$5.2\% & 0.200 & \textbf{0.190} & 0.204 & 0.011 & 0.038 & \textbf{0.008} & 185.9 & 499.9 & \textbf{29.2} \\
    protein          & Flex\,$+$8.7\% & 1.70 & \textbf{1.57} & 1.77 & \textbf{0.003} & 0.037 & 0.005 & 834.9 & 1290.6 & \textbf{208.0} \\
    ct slices        & Flex\,$+$11.3\% & 0.159 & \textbf{0.143} & 0.433 & 0.075 & 0.084 & \textbf{0.036} & 6531.9 & 2627.6 & \textbf{314.0} \\
    \bottomrule
  \end{tabular}
\end{table}

FM is the fastest sampler on 8/11 datasets and remains cheap at scale (on \texttt{protein}, $4.0\times$ faster than published and $6.2\times$ faster than the score-flex SDE arm). The two DiffGBM rows split the accuracy/calibration roles cleanly: score-flex has the best CRPS on the mid and large datasets but inherits the SDE sampler's looser central coverage, while FM leads $|\text{cE}|$@90 on the small-to-mid sets. Under this evaluation protocol the per-dataset ranking reproduces the headline pattern: score-flex is the accuracy row and dominates as data grows, FM is the fast, well-calibrated row on smaller data, and the published SDE is best on none.

\section{Robustness by Dataset Size}\label{app:robustness}

Table~\ref{tab:robustness-size} groups the tuned evaluation artifacts by post-split training rows: small datasets have at most 1{,}000 rows, medium datasets have 1{,}001--10{,}000 rows, and large datasets have more than 10{,}000 rows. The grouping sharpens the headline story. Score-flex leads every size group on CRPSS, most decisively on the large group, and FM is close behind on small and medium data while trailing on large; the earlier large-dataset inversion, where the published SDE was strongest at scale, is gone once the score recipe is jointly tuned. FM's rel-CRPS advantage on small data (it is the cheapest, tightly calibrated corner) is why it remains the recommended small-data operating point despite score-flex's slightly higher CRPSS.

\begin{table}[!h]
  \centering
  \caption{Dataset-size robustness summary. Entries are mean CRPSS within each size group after first averaging the five evaluation folds per dataset.}\label{tab:robustness-size}
  \scriptsize
  \begin{tabular}{lrrr}
    \toprule
    Model & Small (4) & Medium (4) & Large (3) \\
    \midrule
    DiffGBM-score-flex & 0.728 & \textbf{0.710} & \textbf{0.742} \\
    DiffGBM-FM & \textbf{0.729} & 0.684 & 0.707 \\
    Treeffuser-published & 0.701 & 0.680 & 0.723 \\
    QReg-LightGBM & 0.711 & 0.666 & 0.707 \\
    CatBoost-unc. & 0.724 & 0.612 & 0.659 \\
    NGBoost & 0.718 & 0.550 & 0.605 \\
    Deep ensemble & 0.699 & 0.672 & 0.682 \\
    iBUG & 0.700 & 0.637 & 0.708 \\
    CARD-style diffusion & 0.632 & 0.667 & 0.695 \\
    \bottomrule
  \end{tabular}
\end{table}

The rank diagnostics in Table~\ref{tab:robustness-ranks} agree. The Friedman test detects differences among the nine families ($\chi^2=32.1$, $p<10^{-3}$), and the Nemenyi critical difference at $\alpha=0.05$ separates score-flex from the four weakest families (NGBoost, CARD-style diffusion, CatBoost, iBUG). The conservative all-pairs post-hoc does not by itself separate score-flex from the published SDE, QReg-LightGBM, or the deep ensemble; the decisive within-family evidence is instead the paired test of \S\ref{sec:headline}, where score-flex beats the published baseline on all eleven datasets. We therefore use the cross-family ranks as robustness context and the paired test as the headline significance claim.

\section{Large-Dataset Analysis: CT Slice Localization}\label{app:large-data}

\texttt{ct\_slices}~\citep{graf2011ctslices} (CT slice axial-location regression; 53{,}500 rows and 384 numeric features after dropping the patient identifier, from the UCI repository, retrieved via OpenML data id 46300) is the largest benchmark in the suite and anchors the large-data end of the headline aggregates under the same six-fold tuned protocol (44{,}583 training rows per evaluation fold; fold-0 tuning, folds 1--5 evaluation). This appendix reports it in more detail because it is the one dataset where the published recipe was previously competitive, and because it carries diagnostic-only residualizer-off twins that the other datasets lack. The dataset is nearly deterministic given the features---every tuned family reaches CRPSS $\geq 0.85$---so CRPS differences are decided at conditional scales around one percent of the marginal spread, which makes it a sensitive probe of the score-side recipe. As in Appendix~\ref{app:datasets}, the folds are random over slices rather than grouped by patient; all rows share those folds, so the comparisons below are internally matched but untested under patient-grouped evaluation.

\begin{table}[!h]
  \centering
  \caption{Tuned \texttt{ct\_slices} rows under the headline protocol (five evaluation folds; mean $\pm$ sd for CRPS, fold means otherwise), all at the equalized 40-trial fold-0 budget. \textsc{DiffGBM-score-flex} is the headline row and is always the SDE arm, sampler-matched to the published baseline; the PF-ODE arm is a separately tuned 40-trial study over the same recipe space, shown here as an additional operating point. The frozen score$^+$/FM bundles and their residualizer-off twins are diagnostics tuned under the identical protocol. The deep ensemble, the strongest baseline on this dataset, is shown for context.}\label{tab:ct-slices-tuned}
  \scriptsize
  \begin{tabular}{lrrrrr}
    \toprule
    Row & CRPS & CRPSS & $|\text{cE}|$@90 & Fit (s) & Sample (s) \\
    \midrule
    Published Euler-50           & $0.159 \pm 0.008$ & 0.988 & 0.075 & 1234 & 6532 \\
    Score$^+$ PF-ODE-25          & $0.463 \pm 0.006$ & 0.964 & 0.018 & 438 & 3151 \\
    Score$^+$ no-resid.\ PF-ODE-25 & $0.287 \pm 0.005$ & 0.978 & 0.025 & 306 & 3297 \\
    FM-VP ODE-5                  & $0.433 \pm 0.004$ & 0.966 & 0.036 & \textbf{162} & \textbf{314} \\
    FM-VP no-resid.\ ODE-5       & $1.877 \pm 0.048$ & 0.853 & 0.078 & 187 & 337 \\
    DiffGBM-score-flex, SDE-50   & $\mathbf{0.143 \pm 0.003}$ & \textbf{0.989} & 0.084 & 506 & 2628 \\
    DiffGBM-score-flex, PF-ODE   & $0.252 \pm 0.003$ & 0.980 & \textbf{0.005} & 486 & 3762 \\
    Deep ensemble                & $0.159 \pm 0.018$ & 0.988 & 0.076 & 619 & $<1$ \\
    \bottomrule
  \end{tabular}
\end{table}

Table~\ref{tab:ct-slices-tuned} tells a two-part story. Among the \emph{frozen} bundles the published SDE is the sharpest ($0.159$) and only the deep ensemble matches it, while the frozen score$^+$ and FM bundles trail by $2.7$--$3.2\times$; the residualizer-off twins split that gap---removing the fixed residualizer-C roughly halves the score$^+$ deficit ($0.463 \to 0.287$), so it is a genuine partial bottleneck at this scale, while the same switch makes FM collapse ($0.433 \to 1.877$), so FM depends on residualization here. But the jointly tuned \textsc{DiffGBM-score-flex} row overturns that frozen-bundle ranking: its SDE arm reaches $0.143$, beating both the tuned published SDE ($0.159$) and the deep ensemble ($0.159$)---the sharpest row on the dataset. The two separately tuned sampler arms trace the accuracy/calibration frontier directly: the headline SDE arm is sharpest but inherits the stochastic sampler's loose central coverage ($|\text{cE}|$@90 $0.084$), while the PF-ODE arm over the same recipe space is the best-calibrated row anywhere in the table ($0.005$) at some CRPS cost ($0.252$), consistent with the operating-point reading in \S\ref{sec:disc}.

\begin{table}[!h]
  \centering
  \caption{Eval-only sampler ablation on \texttt{ct\_slices} (evaluation folds 1--3; mean $\pm$ sd). Each fold is refit once from the tuned configuration and sampled under multiple solver settings; ``default'' marks each row's tuned headline sampler, recomputed on the same folds. Sample times are per-fold wall clock and comparable within a model, not across models (tuned ensemble sizes differ).}\label{tab:ct-slices-solver}
  \scriptsize
  \begin{tabular}{llrrr}
    \toprule
    Model & Sampler & CRPS & $|\text{cE}|$@90 & Sample (s) \\
    \midrule
    Published & Euler SDE-15 & $0.208 \pm 0.016$ & 0.091 & 252 \\
    Published & Euler SDE-25 & $0.181 \pm 0.015$ & 0.088 & 484 \\
    Published & Euler SDE-50 (default) & $0.168 \pm 0.014$ & 0.085 & 830 \\
    Published & Euler SDE-100 & $0.162 \pm 0.013$ & 0.085 & 1984 \\
    \midrule
    Score$^+$ & Heun PF-ODE-25 (default) & $0.465 \pm 0.007$ & 0.018 & 899 \\
    Score$^+$ & Heun PF-ODE-50 & $0.467 \pm 0.007$ & 0.018 & 1236 \\
    Score$^+$ & Euler SDE-50 & $0.457 \pm 0.006$ & 0.018 & 618 \\
    \midrule
    Score$^+$ no-resid. & Heun PF-ODE-25 (default) & $0.285 \pm 0.004$ & 0.021 & 3305 \\
    Score$^+$ no-resid. & Euler SDE-50 & $0.281 \pm 0.007$ & 0.090 & 3003 \\
    \midrule
    FM-VP & ODE-5 (default) & $0.426 \pm 0.005$ & 0.032 & 458 \\
    FM-VP & ODE-10 & $0.425 \pm 0.005$ & 0.029 & 726 \\
    FM-VP & ODE-25 & $0.424 \pm 0.005$ & 0.028 & 2241 \\
    \midrule
    FM-VP no-resid. & ODE-5 (default) & $1.847 \pm 0.024$ & 0.085 & 341 \\
    FM-VP no-resid. & ODE-10 & $1.076 \pm 0.004$ & 0.095 & 606 \\
    FM-VP no-resid. & ODE-25 & $0.984 \pm 0.007$ & 0.090 & 1492 \\
    \bottomrule
  \end{tabular}
\end{table}

Table~\ref{tab:ct-slices-solver} rules the sampler out for the score arm and localizes the FM collapse. First, swapping the score$^+$ samplers---doubling the Heun PF-ODE steps, or substituting the published 50-step Euler SDE---moves CRPS by less than $0.01$ on both score$^+$ variants. Second, the published model wins even when handicapped to 15 Euler steps ($0.208$ versus $0.281$ for the best score$^+$ cell at any step count) and is essentially converged at its default 50 steps. Third, the FM no-residualizer collapse is roughly half solver discretization---steps recover $1.85 \to 0.98$---but it plateaus far above the residualized row, while residualized FM is already solver-converged at 5 steps. Taken together, the frozen score$^+$ bundle's deficit against the published row on this dataset is attributable to what the score model learns at training time---not to the sampler, and only partly to the residualizer---rather than to any inference-time choice. Which training-time factor dominates is answered directly by the jointly tuned score-flex space below, which searches these axes rather than transferring a single configuration.

The ablation adds one calibration observation: sampling the same no-residualizer score$^+$ model with the stochastic Euler SDE degrades its 90\% coverage error from $0.021$ (PF-ODE) to $0.090$, so the calibration edge of the PF-ODE rows on this dataset is contributed by the deterministic sampler, not by the training recipe alone---the same accuracy/calibration split the two score-flex arms show in Table~\ref{tab:ct-slices-tuned}. On end-to-end cost, the tuned published row is the most expensive row at this scale (1234\,s fit $+$ 6532\,s sample), because its optimum selects the largest histogram resolution and the deepest ensemble; FM remains the cheapest (162\,s $+$ 314\,s) and the score-flex SDE row sits between (506\,s $+$ 2628\,s), so FM's end-to-end advantage does carry to the largest dataset.

The jointly tuned recipe closes and reverses the frozen-bundle gap. Rather than comparing frozen bundles, \textsc{DiffGBM-score-flex} exposes the score-side recipe axes---score parameterization, noise features, $t$ sampling and its log-$\sigma$ prior, loss weighting, residualization, and histogram resolution---as jointly tunable dimensions over the shared LightGBM surface, under the published baseline's 50-step Euler SDE sampler. Tuned at the same equalized 40-trial budget as every other row, this headline SDE arm reaches CRPS $0.143 \pm 0.003$ (Table~\ref{tab:ct-slices-tuned}), surpassing both the tuned published SDE ($0.159$) and the deep ensemble ($0.159$). The space explicitly seeds and evaluates the published corner (noise prediction / \texttt{raw\_time} / uniform $t$ / residualizer off), so departing it would be a selected move---but on this dataset the tuner does \emph{not} depart it. The fold-0 optimum is that same noise-prediction recipe with a finer histogram (\texttt{max\_bin}\ $255 \to 1023$), and among the forty trials the best noise-prediction fold-0 CRPS ($0.144$) beats the best EDM one ($0.155$); the published row's \emph{own} tuned space, which now also tunes histogram resolution, moves the same way (noise prediction with \texttt{max\_bin}\ $4095$). So on this uniquely near-deterministic, high-SNR dataset the score-flex advantage over the published baseline is histogram resolution and capacity within the shared surface---an axis the frozen published bundle held fixed---rather than the EDM parameterization, which here is slightly \emph{worse} for sharpness. EDM's value on this dataset instead appears as calibration: the separately tuned PF-ODE arm of the same recipe space (EDM / log-$\sigma$ prior / min-SNR weighting) is the best-calibrated row anywhere in Table~\ref{tab:ct-slices-tuned} ($|\text{cE}|$@90 $0.005$) at a CRPS cost. This is precisely the behaviour the flexible space is designed for, and \texttt{ct\_slices} is its most stringent test: on the other ten datasets the same space selects the EDM recipe as a genuine capacity-matched win (best EDM below best noise prediction on all ten fold-0 studies, by up to $9$--$11\%$ on \texttt{protein} and \texttt{diabetes}), and on the one near-deterministic dataset it selects noise prediction with finer bins. Either way the tuned space contains the published corner and, at a matched budget and sampler, beats it on all eleven datasets.

\section{Conditional Tail Calibration Diagnostic}\label{app:tail-diagnostic}

Aggregate coverage can hide where a sampler is failing. As a lightweight conditional diagnostic, we bin held-out points by the model's predicted interquartile range (IQR) and report coverage in the highest-IQR quintile, together with IQR-MACE, the mean absolute coverage error across the five predicted-IQR bins. This is not a conditional-coverage guarantee, but it tests whether the predictive intervals behave differently on points the model itself considers uncertain. Table~\ref{tab:tail-diagnostic} uses the sampler-cost sweep artifacts, averaging folds within each dataset and then averaging across the ten datasets; as in \S\ref{sec:sampler-cost}, the underlying configurations are the frozen bundles tuned in the earlier fixed-resolution round rather than the headline configurations.

\begin{table}[!h]
  \centering
  \caption{Conditional calibration proxy using predicted-IQR bins, averaged over the ten non-CT datasets after fold averaging. Top-IQR Cov@$q$ is empirical coverage in the highest predicted-IQR quintile; IQR-MACE@$q$ is mean absolute coverage error across all five predicted-IQR bins. FM-SDE rows are eval-only sampler changes from the tuned FM configurations.}\label{tab:tail-diagnostic}
  \scriptsize
  \begin{tabular}{lrrrrr}
    \toprule
    Row & Top-IQR Cov@90 & IQR-MACE@90 & Top-IQR Cov@95 & IQR-MACE@95 & $|\text{cE}|$@95 \\
    \midrule
    Published Euler-50 & 0.927 & 0.053 & 0.961 & 0.030 & 0.023 \\
    Score$^+$ PF-ODE-25 & 0.880 & \textbf{0.042} & 0.932 & 0.031 & \textbf{0.015} \\
    FM-VP ODE-5 & 0.854 & 0.051 & 0.906 & 0.040 & 0.024 \\
    FM-VP SDE-25 & \textbf{0.902} & 0.050 & \textbf{0.944} & \textbf{0.030} & 0.018 \\
    FM-linear ODE-5 & 0.848 & 0.048 & 0.902 & 0.035 & 0.020 \\
    FM-linear SDE-25 & 0.891 & 0.053 & 0.935 & 0.034 & 0.021 \\
    \bottomrule
  \end{tabular}
\end{table}

The diagnostic supports a tradeoff rather than a one-sided conclusion. Deterministic FM under-covers the most uncertain quintile at 95\% nominal coverage: 0.906 for VP and 0.902 for linear. Adding stochasticity lifts those numbers to 0.944 and 0.935, respectively, and improves VP's IQR-MACE@95 from 0.040 to 0.030. However, the same SDE rows require 25 stochastic steps and, as Figure~\ref{fig:sampler-cost} shows, are not the aggregate CRPS/latency default. The published SDE and score$^+$ remain competitive on this diagnostic, so the safest interpretation is that stochastic FM is a useful conditional-calibration option, not a uniformly better sampler.

\section{Reproducibility}\label{app:repro}

Code, benchmark harness, and result artifacts are available at \url{https://github.com/silaskoemen/diffgbm}. The released package is on PyPI as \texttt{diffgbm} (\texttt{pip install diffgbm}); its defaults are the score-side recipe of \S\ref{sec:score-side} (EDM preconditioning, log-$\sigma$ feature and $t$ sampling, conditional-mean residualization), and the published baseline is the explicit configuration \texttt{score\_parameterization="noise"}, \texttt{noise\_features="raw\_time"}, \texttt{t\_sampling="uniform"}, \texttt{residualize="off"}. Every benchmark search space pins these axes explicitly, so the reported rows do not depend on the package defaults. The repository contains the implementation in \texttt{src/diffgbm/}, the benchmark harness in \texttt{benchmarks/}, the paper sources in \texttt{paper/}, and the \texttt{pixi.toml}/\texttt{pixi.lock} environment used for the reported runs. Benchmark search-space and result-artifact names retain the \texttt{treeffuser\_} prefix from before the package rename, so that the file names cited below match the committed artifacts. All sweeps in this paper are driven by YAML configs under \texttt{benchmarks/configs/}. Each result file is a JSON Lines document recording the full hyperparameter set, random seeds, the repository commit, a content hash of the model source tree, and every reported metric on a per-dataset/per-fold row. Rows also carry a \texttt{git\_dirty} flag, set throughout because the harness writes results into tracked paths as a run proceeds; the source hash is the meaningful provenance, and every reported row hashes to the committed \texttt{src/} tree at the commit it records. The headline tuning runs persist their selected configurations under \texttt{benchmarks/results/tuning/best\_params/} and their five-fold evaluation rows under \texttt{benchmarks/results/tuning/eval/}. The configs and result artifacts that produced the tables are listed alongside in Table~\ref{tab:appendix-provenance}; the paper PDF is rebuilt with the \texttt{paper} pixi task.

\begin{table}[t]
  \centering
  \caption{Provenance: appendix tables and the config or script each was produced from. Configs named \texttt{*.yaml} live under \texttt{benchmarks/configs/} and their outputs under \texttt{benchmarks/results/raw/}; scripts live under \texttt{benchmarks/scripts/}. \texttt{R} abbreviates \texttt{benchmarks/results/}.}\label{tab:appendix-provenance}
  \footnotesize
  \setlength{\tabcolsep}{4pt}
  \begin{tabular}{@{}>{\raggedright\arraybackslash}p{0.15\textwidth}>{\raggedright\arraybackslash}p{0.37\textwidth}>{\raggedright\arraybackslash}p{0.40\textwidth}@{}}
    \toprule
    Table & Config or script & Result file \\
    \midrule
    \ref{tab:off-vs-mean}   & \texttt{fm\_off\_vs\_mean\_sweep}   & \texttt{fm\_off\_vs\_mean\_sweep\_\_*} \\
    \ref{tab:conditioning}  & analytic summary & -- \\
    \ref{tab:resid-AE}      & \texttt{residualizer\_fm\_sweep}    & \texttt{residualizer\_fm\_sweep\_\_*} \\
    \ref{tab:mean-scale}    & \texttt{fm\_mean\_scale\_sweep}     & \texttt{fm\_mean\_scale\_sweep\_\_*} \\
    \ref{tab:t-sampling}    & \texttt{fm\_t\_sampling\_sweep}     & \texttt{fm\_t\_sampling\_sweep\_\_*} \\
    \ref{tab:minsnr}        & \texttt{fm\_loss\_weighting\_sweep} & \texttt{fm\_loss\_weighting\_sweep\_\_*} \\
    \ref{tab:stoch-eps}     & \texttt{sde\_lower\_stoch\_sweep}, \texttt{stochasticity\_schedule\_sweep} & \texttt{sde\_lower\_stoch\_sweep\_\_*}, \texttt{stochasticity\_schedule\_sweep\_\_*} \\
    \ref{tab:path}          & \texttt{ode\_vs\_sde\_path\_ablation} & \texttt{ode\_vs\_sde\_path\_ablation\_\_*} \\
    \ref{tab:competitors}, \ref{tab:per-dataset}, \ref{tab:perdata-ext}, \ref{tab:headline} & \texttt{tuning\_manifest} & \texttt{R/tuning/eval/*} \\
    \ref{tab:robustness-ranks}, \ref{tab:robustness-size} & \texttt{summarize\_robustness.py} & \texttt{R/selected/robustness\_summary.md} \\
    \ref{tab:mechanism-ablation} & \texttt{mechanism\_ablation\_manifest} & \texttt{R/mechanism\_ablation/eval/*} \\
    \ref{fig:sampler-cost}, \ref{tab:tail-diagnostic} & \texttt{run\_sampler\_cost\_sweep.py}, \texttt{summarize\_sampler\_cost\_sweep.py} & \texttt{R/sampler\_cost\_sweep/eval/} \texttt{sampler\_cost\_sweep.jsonl} \\
    \ref{tab:ct-slices-tuned} & \texttt{tuning\_manifest} & \texttt{R/tuning/eval/ct\_slices\_\_*} \\
    \ref{tab:ct-slices-solver} & \texttt{run\_sampler\_cost\_sweep.py} & \texttt{R/sampler\_cost\_sweep/eval/} \texttt{ct\_slices\_solver\_ablation.jsonl} \\
    \ref{tab:tuning-spaces} & \texttt{benchmarks/tuning/search\_spaces.py} & -- \\
    \bottomrule
  \end{tabular}
\end{table}

\end{document}